\documentclass{article}

\usepackage{hyperref}
\usepackage{url}
\usepackage{microtype}
\usepackage{graphicx}
\usepackage{subcaption}
\usepackage{tabularx}
\usepackage{multicol, multirow}
\usepackage{makecell}
\usepackage{verbatim}
\usepackage{xcolor}
\usepackage{amssymb} 
\usepackage{pifont}  
\usepackage[table]{xcolor}
\usepackage{enumitem}
\usepackage{array}
\usepackage{caption}
\usepackage{adjustbox}
\usepackage{booktabs} 
\usepackage[normalem]{ulem}
\useunder{\uline}{\ul}{}

\newcolumntype{Y}{>{\centering\arraybackslash}X}

\usepackage{hyperref}

\usepackage[preprint]{icml2026}

\usepackage{amsmath}
\usepackage{amssymb}
\usepackage{mathtools}
\usepackage{amsthm}

\usepackage{amsmath,amsfonts,bm, amsthm}

\newcommand{\ourmethod}{\textsc{CSMC}}
\newcommand{\methodfullname}{Clean-Sample Markov Chain Sampler}

\newcommand{\newterm}[1]{{\bf #1}}

\newcommand{\CC}{\cellcolor{gray!12}}

\newcommand{\cat}[2]{\text{Cat}\left( #1 ; #2\right)}

\newcommand{\cmark}{\textcolor{black}{\checkmark}} 
\newcommand{\xmark}{\textcolor{black}{\ding{55}}}   

\def\eqref#1{equation~\ref{#1}}

\def\1{\bm{1}}

\def\mQ{{\bm{Q}}}
\def\mR{{\bm{R}}}

\DeclareMathAlphabet{\mathsfit}{\encodingdefault}{\sfdefault}{m}{sl}
\SetMathAlphabet{\mathsfit}{bold}{\encodingdefault}{\sfdefault}{bx}{n}

\def\gU{{\mathcal{U}}}

\def\gX{{\mathcal{X}}}

\newcommand{\E}{\mathbb{E}}

\newcommand{\R}{\mathbb{R}}

\newcommand{\KL}{D_{\mathrm{KL}}}

\usepackage[capitalize,noabbrev]{cleveref}

\theoremstyle{plain}
\newtheorem{theorem}{Theorem}[section]

\theoremstyle{definition}

\theoremstyle{remark}

\usepackage[textsize=tiny]{todonotes}

\icmltitlerunning{Clean-Sample Markov Chain Sampler}

\begin{document}

\twocolumn[
  \icmltitle{Reward-Guided Discrete Diffusion via Clean-Sample Markov Chain for\\Molecule and Biological Sequence Design}



  \icmlsetsymbol{equal}{*}

  \begin{icmlauthorlist}
    \icmlauthor{Prin Phunyaphibarn}{kaist}
    \icmlauthor{Minhyuk Sung}{kaist}
  \end{icmlauthorlist}

  \icmlaffiliation{kaist}{KAIST, Daejeon, South Korea}

  \icmlcorrespondingauthor{Prin Phunyaphibarn}{prin10517@kaist.ac.kr}
  \icmlcorrespondingauthor{Minhyuk Sung}{mhsung@kaist.ac.kr}

  \icmlkeywords{Machine Learning, ICML}

  \vskip 0.3in
]



\printAffiliationsAndNotice{}  

\begin{abstract}
Discrete diffusion models have recently emerged as a powerful class of generative models for chemistry and biology data. In these fields, the goal is to generate various samples with high rewards (e.g., drug-likeness in molecules), making reward-based guidance crucial. Most existing methods are based on guiding the diffusion model using intermediate rewards but tend to underperform since intermediate rewards are noisy due to the non-smooth nature of reward functions used in scientific domains. To address this, we propose Clean-Sample Markov Chain (\ourmethod) Sampler, a method that performs effective test-time reward-guided sampling for discrete diffusion models, enabling local search without relying on intermediate rewards. \ourmethod{} constructs a Markov chain of clean samples using the Metropolis-Hastings algorithm such that its stationary distribution is the target distribution. We design a proposal distribution by sequentially applying the forward and backward diffusion processes, making the acceptance probability tractable. Experiments on molecule and biological sequence generation with various reward functions demonstrate that our method consistently outperforms prior approaches that rely on intermediate rewards.
\end{abstract}
\vspace{-\baselineskip}
\section{Introduction}
\label{sec:intro}
\vspace{-0.5\baselineskip}
Discrete diffusion models have recently emerged as a powerful generative framework for discrete data, showing particular promise in chemistry and biology for generating complex structures such as molecules and DNA sequences. 
Unlike autoregressive models that assume a canonical left-to-right ordering on the data, discrete diffusion models are more naturally suited for scientific domains whose data (e.g., molecules or DNA sequences) lack a natural fixed ordering. For instance, the widely used SMILES~\citep{weininger1988smiles} representation for molecules is based on heuristic rules such as depth-first search which does not use a unified fixed ordering~\citep{leegenmol}.

In many applications in chemistry and biology, there are well-defined notions of \emph{quality} for the data; for example, drug-likeness~\citep{bickerton2012qed} for molecular structures or enhancer activity~\citep{taskiran2024cell, wangdrakes} for DNA sequences. Thus, generative models must not only produce \emph{natural} samples that resemble the training data but also achieve high-quality scores according to such domain-specific criteria.

In the emerging regime of test-time scaling, quality considerations are incorporated by defining a reward function and optimizing it during inference through reward-guided sampling. The simplest strategy is Best-of-N sampling, where $N$ samples are generated, and the one with the highest reward is selected. However, this brute-force method is inefficient since it does not perform any structured guidance. Recent works~\citep{kimtestsmc, wu2023tds, yu2023freedom, bansal2023universal, chungdps} have instead proposed guidance using \emph{process rewards} or \emph{intermediate rewards}, which are computable at intermediate steps of generation. 

While methods involving intermediate rewards are technically applicable to discrete diffusion models, they pose particular challenges for many types of chemistry and biology data as reward functions in these domains are often non-smooth, meaning that small perturbations in the data can cause large changes in the reward. For example, in molecular structures, modifying even a single element in the string representation can render the entire molecule invalid, collapsing the reward to zero, as shown in Fig.~\ref{fig:method} (left). As a result, relying on intermediate rewards does not provide an effective local search strategy in these cases.

The key question that arises here is how to enable local search in reward-guided generation without relying on intermediate rewards. We introduce \textbf{Clean-Sample Markov Chain (\ourmethod) Sampler} which performs iterative search over clean data samples using the Metropolis–Hastings (MH) algorithm. To propose a new clean sample from an existing one, we use forward–backward combinations---applying the forward process to corrupt clean data followed by running the reverse process to obtain a new sample. While the acceptance probability for the MH algorithm is intractable due to intractable clean sample probabilities, we show that the forward-backward proposal distribution makes the acceptance probability tractable, enabling efficient sampling.

We validate \ourmethod{} on molecule and biological sequence generation across four different reward functions. \ourmethod{} achieves the highest reward in all settings, even for SMILES string generation where other methods fail due to inaccurate intermediate rewards.
\begin{figure*}[t]
    \centering
    \includegraphics[width=\linewidth]{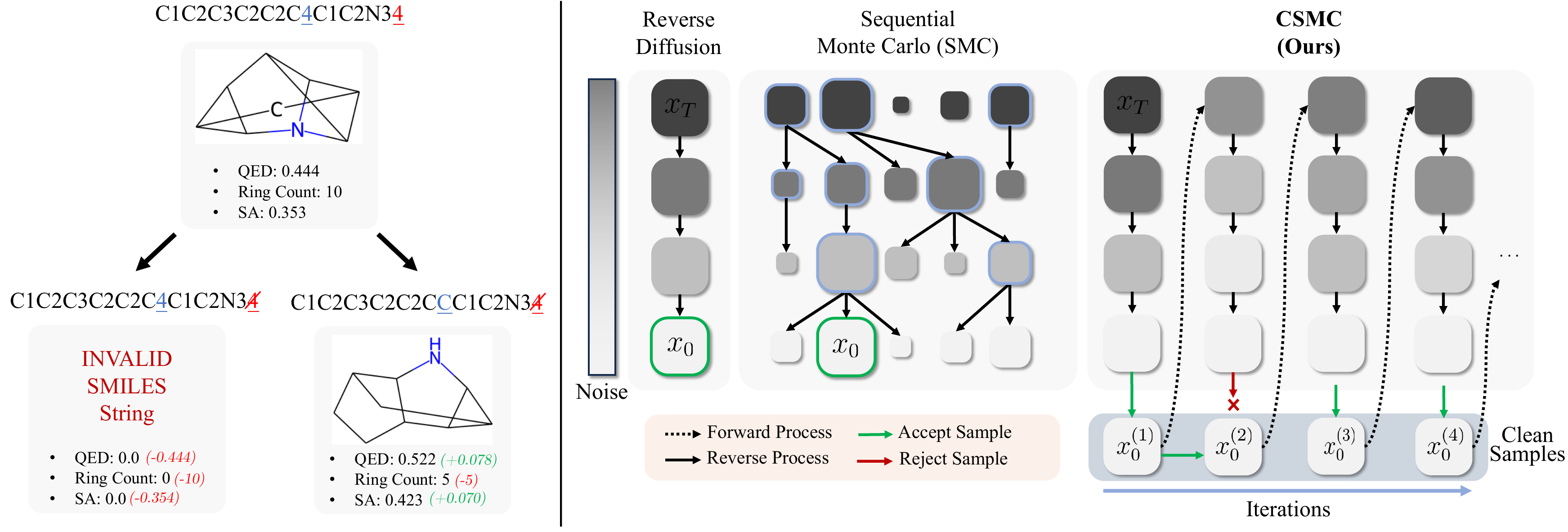}
    \vspace{-\baselineskip}
    \caption{\textbf{Left:} In scientific applications, the rewards defined on discrete spaces are highly sensitive to small perturbations. A one-character change to a SMILE string can result in an invalid string with zero reward. Properties such as QED, ring count, and synthetic accessibility (SA) can also vary significantly even when changing only one or two tokens. \textbf{Right:} Reverse diffusion and typical inference-time scaling methods~\citep{kimtestsmc,li2024svdd} such as SMC~\citep{kimtestsmc} rely on guiding samples through noise levels by constructing a Markov chain beginning with pure noise and ending with clean samples. Our \ourmethod{} constructs a Markov chain consisting of only clean samples by successively applying the forward and reverse processes sequentially at each step. This formulation bypasses the need for intermediate rewards by evaluating the reward directly on clean samples while leveraging information from past samples for guidance.}
    \vspace{-\baselineskip}
    \label{fig:method}
\end{figure*}
\section{Related Work}
\label{sec:related-work}
\vspace{-0.5\baselineskip}
In continuous diffusion models, reward-guided sampling is conventionally performed using gradient-based methods~\citep{dhariwal2021cg, ho2022cfg, chungdps, song2023pigdm, rozet2024momentmatching, bansal2023universal, yoon2025psi, kimtestsmc, wu2023tds} which offer strong guidance towards high-reward regions. However, gradient-based approaches cannot be applied to discrete diffusion models as gradients are ill-defined in discrete spaces and it is not theoretically valid to add a continuous gradient to a discrete objective.

\begin{table}
{
\small
\begin{tabularx}{\linewidth}{Y | >{\centering\arraybackslash}m{0.065\textwidth} >{\centering\arraybackslash}m{0.07\textwidth} >{\centering\arraybackslash}m{0.06\textwidth} >{\centering\arraybackslash}m{0.1\textwidth}}
\toprule
& \makecell{Uniform} & \makecell{Masked} & \makecell{Clean\\Reward} & \makecell{Trajectory\\Guidance} \\
\midrule
BoN & \cmark & \cmark & \cmark & \xmark \\
SMC & \cmark & \cmark & \xmark & \cmark \\
SVDD & \cmark & \cmark & \xmark & \cmark \\
SGDD & \cmark & \xmark & \xmark & \cmark \\
\rowcolor{gray!12} \textbf{\ourmethod} & \cmark & \cmark & \cmark & \cmark \\
\bottomrule
\end{tabularx}
}
\caption{Comparison of discrete diffusion inference-time scaling methods for reward-guided sampling. Our~\ourmethod{} applies to all discrete diffusion frameworks and leverages the clean reward while guiding the sampling trajectory.}
\vspace{-\baselineskip}
\label{tab:method-comparison}
\end{table}

\vspace{-0.5\baselineskip}
\paragraph{Training-Free Reward-Guided Sampling for Discrete Diffusion.} 
Recently, inference-time scaling methods for discrete diffusion models have been proposed to tackle reward-guided sampling. A comparison of these methods and our method is shown in Tab.~\ref{tab:method-comparison}.

The simplest method is Best-of-N (BoN) sampling~\citep{stiennon2020bon}, which generates $N$ samples independently and selects the one with the highest reward. Due to its simplicity, BoN is applicable to all types of discrete diffusion models. However, BoN does not guide the denoising trajectory using the reward, resulting in inefficient search, especially when high-reward samples lie in low-density regions unlikely to be sampled by the model. 

On the other hand, particle-based methods~\citep{ma2025scaling, kim2025rbf, singhalfksteering, li2024svdd}, such as SMC~\citep{doucet2001smcintro, naesseth2019elementsofsmc, moral2004fksmc} and SVDD~\citep{li2024svdd}, can be applied to discrete diffusion and incorporate reward signals \emph{during} the denoising process through a reward-based resampling step. 
SMC~\citep{doucet2001smcintro, naesseth2019elementsofsmc, moral2004fksmc}, takes $N$ particles (samples) at each step and samples a subset of the particles to keep while throwing the rest away. By adjusting the probability of keeping each particle based on its expected reward, SMC encourages local exploration around potentially promising samples. SVDD~\citep{li2024svdd} exploits local search around high-reward samples by generating multiple candidate samples at each timestep and retaining only one sample by sampling the candidates with probability proportional to the reward. A key challenge with these methods is that they require computing \newterm{intermediate rewards} on noisy samples rather than on fully denoised samples. To address this,  these approaches exploit the diffusion model’s ability to predict an approximation of the clean sample $x_0$ given a noisy sample $x_t$, and compute the rewards on this $x_0$ prediction instead. This enables reward-guidance \emph{during} the denoising process, but also inherently assumes that the predicted $x_0$ is a good approximation of the true clean sample since exploration will be performed locally around samples with high intermediate rewards, an assumption which does not hold in many scientific applications. 

Recently, \citet{chu2025sgdd} proposed SGDD which uses the split Gibbs sampler~\citep{vono2019splitgibbs} by alternating between denoising steps and running MCMC to optimize intermediate noisy samples. However, SGDD applies \emph{only} to uniform diffusion models where intermediate noisy samples can be treated as clean samples. The characteristics of uniform discrete diffusion is discussed in more detail in Sec.~\ref{sec:method}. Furthermore, SGDD still relies on computing the reward on the intermediate noisy samples during the MCMC optimization step.

In this work, we bypass the need for intermediate rewards by using the Metropolis-Hastings algorithm~\citep{robert2009mh} to construct a Markov chain of \emph{clean} samples that converges to the desired target distribution. Our method is applicable to both uniform and masked discrete diffusion models, and only requires the computation of \newterm{clean rewards} on clean samples, which can then be used to efficiently guide the Markov chain (Tab.~\ref{tab:method-comparison}).

\vspace{-\baselineskip}
\paragraph{Training-Based Reward-Guided Sampling for Discrete Diffusion.} Unlike continuous diffusion, discrete diffusion models cannot leverage gradient signals due to the non-differentiable nature of discrete spaces. As such, only gradient-free approaches can be applied to discrete diffusion models. Guidance methods~\citep{nisonoffdiscreteguidance, schiff2024udlm} have also been developed for diffusion models but require training a classifier on noisy data. Due to its non-differentiability, RLFT approaches are often used in discrete diffusion~\citep{rectorddpp, wangdrakes, zekri2025sepo} but are difficult to train since non-differentiability forces fine-tuning to be done using policy gradients based on highly complex reward landscapes~\citep{uehara2025inference}. Crucially, modifying the model weights also runs the risk of deviating from the pretrained prior, affecting the naturalness of the generated samples and further complicating training.
\section{Background}
\label{sec:background}
\vspace{-0.5\baselineskip}
Diffusion models~\citep{sohl2015deep, ho2020ddpm, song2020score} learn to reverse a forward Markov process. Although originally designed for continuous spaces, diffusion models have also been successfully applied to discrete state spaces~\citep{austin2021d3pm, campbell2022ctmc, lou2024sedd, sahoo2024mdlm, schiff2024udlm, shi2024md4}.

\subsection{Discrete and Continuous-Time Discrete Diffusion}
\vspace{-0.5\baselineskip}
\paragraph{Discrete-Time Discrete Diffusion.}
These discrete diffusion models are characterized by forward transition matrices $\mQ_t$. Let $\overline{\mQ}_t = \mQ_1 \mQ_2 \cdots \mQ_t$. These transition matrices define the forward process:
\begin{align*}
    p_{t|t-1}(x_t | x_{t-1}) &= \cat{x_t}{p = x_{t-1}\mQ_t}, \\
    p_t(x_t | x_0) &= \cat{x_t}{p = x_0 \overline{\mQ}_t},
\end{align*}
where $\text{Cat}(\cdot;p)$ denotes the categorical distribution with probabilities given by $p$. 

While the reverse process probability $p(x_{t-1}|x_t)$ is not directly tractable, by additionally conditioning on $x_0$, the reverse process can be derived in closed form as
\begin{align*}
    p(x_{t-1} | x_t, x_0) = \cat{x_{t-1}}{p=\frac{x_t \mQ_t^\top \odot x_0 \overline{\mQ}_{t-1}}{x_0 \overline{\mQ}_t x_t^\top}}.
\end{align*}
A neural network $p_\theta(x_0|x_t) \approx p(x_0|x_t)$ is learned and denoising is performed by the following parameterization:
\begin{align*}
    p_\theta(x_{t-1} | x_t) \propto \sum_{x_0} q(x_{t-1}, x_t | x_0) p_\theta(x_0 | x_t).
\end{align*}
Notably, the learned neural network $p_\theta(x_0|x_t)$ predicts a distribution over $x_0$ given $x_t$ which enables sampling $x_0$-predictions $\hat{x}_0(x_t) \sim p_\theta(x_0 | x_t)$.

\paragraph{Continuous-Time Discrete Diffusion.}
\vspace{-0.5\baselineskip}
\citet{campbell2022ctmc} proposed a continuous-time framework based on Continuous-Time Markov Chains (CTMC) where state transitions can occur at any time. Instead of defining transition matrices, CTMC-based discrete diffusion defines a forward transition \emph{rate} matrix $\mR_t$ which defines the infinitesimal transition probability between two timesteps:
\begin{align*}
    p(x_t | \tilde{x}_{t-\Delta t}) = \delta_{x_t,\tilde{x}_{t-\Delta t}} + \mR_t(\tilde{x}_{t-\Delta t}, x_t) \Delta t + o(\Delta t).
\end{align*}

By using a continuous-time framework, advanced sampling strategies such as predictor-corrector methods~\citep{campbell2022ctmc, zhao2024informedpredictor} and planning~\citep{liuthinkbefore, peng2025p2pathplanning} and score-based approaches~\citep{meng2022concrete, sunsddm, lou2024sedd} have been proposed. 

\subsection{Masked and Uniform Discrete Diffusion}
\vspace{-0.5\baselineskip}
Discrete diffusion models allow the user to choose the transition matrix. The two most common choices of transition matrices result in masked diffusion models (MDMs)~\citep{austin2021d3pm, sahoo2024mdlm, shi2024md4} and uniform state models (USMs)~\cite{austin2021d3pm, schiff2024udlm}.

\vspace{-0.5\baselineskip}
\paragraph{Masked Diffusion Models (MDMs).} MDMs define the forward process by progressively replacing tokens with a special mask token. The denoising model learns to recover the original sequence from these masked inputs. Notably, samples at intermediate time steps are not valid samples because of the mask tokens, which are not present in the dataset.

\vspace{-0.5\baselineskip}
\paragraph{Uniform State Models (USMs).} USMs replace each token with a randomly chosen token from the vocabulary as noise increases. This creates a uniform corruption process in which every other possible token substitution is equally likely. Unlike MDMs, samples at intermediate time steps may constitute valid samples. 
\section{\methodfullname}
\label{sec:method}
\vspace{-0.5\baselineskip}
Let the reward function be denoted by $r(\cdot): \gX \to \R$ where $\gX$ is the domain, and $\Delta(\gX)$ denote the set of all probability distributions defined on $\gX$. Given a reward function $r(x)$, our objective is to generate samples with high rewards while maintaining naturalness by leveraging a pretrained generative model $p^\text{pre}(\cdot)$. To accomplish this, we sample from the reward-weighted distribution $p_\beta(x)$:
\begin{align*}
    p_\beta(x) &:= \underset{p \in \Delta(\gX)}{\arg \max} \E_{x \sim p(\cdot)} \left[ r(x) \right] - \beta \KL(p(\cdot) \| p^\text{pre}(\cdot)) \\
    &\propto \exp (r(x) / \beta) p^\text{pre}(x),
\end{align*}
where $\beta$ is a hyperparameter controlling the ``naturalness'' of the samples through KL-regularization with the pretrained model.

Most previous methods such as particle-based methods perform guidance by utilizing intermediate rewards $r(\hat{x}_0(x_t))$ computed from the $x_0$-prediction $\hat{x}_0(x_t)$ at $x_t$ as an approximation to the clean sample. However, intermediate rewards are often inaccurate and noisy in many scientific applications due to the non-smooth reward functions. In scientific applications, even a slight perturbation can significantly impact the reward. For instance, when generating molecules based on the SMILES~\citep{weininger1988smiles} representation, modifying a single token on a high-reward molecule can result in an invalid molecule with zero reward or a molecule with very different properties, as shown in Fig.~\ref{fig:method} (left). This is in stark contrast to rewards defined on continuous space where a slight perturbation smoothly affects the reward (small perturbations to some pixel values of an image leave the semantics of the image unchanged). Due to the sensitivity of the reward, the $x_0$ prediction must be perfectly accurate, otherwise the intermediate rewards will be uninformative. A sample whose intermediate reward is zero may be prematurely removed by a particle-based method even though changing a single token may yield a high-reward sample (Fig.~\ref{fig:method}, left).

In these cases, it is more beneficial to leverage only clean reward $r(x_0)$ computed on a clean sample $x_0$ rather than on an approximation $\hat{x}_0(x_t)$. One such method leveraging clean rewards is Best-of-N (BoN) sampling~\citep{stiennon2020bon}. However, BoN does not provide guidance during the denoising process, resulting in inefficient exploration. Its effectiveness is therefore limited to what the pretrained model can already generate: if high-reward samples lie in low-density regions of the model’s distribution, they are unlikely to ever be produced, even when many samples are drawn. A natural question arises: \textbf{How can we leverage clean rewards while using information from past samples for guidance?}

Our proposed method, {\ul \bf C}lean-{\ul \bf S}ample {\ul \bf M}arkov {\ul \bf C}hain \ \textbf{(\ourmethod)} Sampler, answers this question by using the Metropolis-Hastings (MH) algorithm to construct a Markov chain of clean samples that converges to the reward-weighted distribution $p_\beta(x_0)$. By using only clean samples, \ourmethod \ leverages accurate clean rewards. By using the MH algorithm to iteratively refine the samples, the clean rewards of past samples can be used for guidance.

\subsection{Bypassing the Intermediate Rewards}
\vspace{-0.5\baselineskip}
Previous methods require intermediate rewards because the denoising trajectory of diffusion models output clean samples only at the very end. A natural solution is to instead explore the clean data space by constructing a chain of \emph{clean} samples $\{x_0^{(i)}\}_i$, which converges to the target distribution $p_\beta(x_0)$. At each iteration, rewards can then be computed directly on clean samples, yielding clean rewards $r(x_0)$ which can then be used to guide the chain towards the target distribution. To accomplish this, we propose to use the Metropolis-Hastings (MH) algorithm.

Suppose we want to sample from the (potentially unnormalized) target distribution $p_\beta(x_0)$. We first define a proposal distribution $q(x'_0|x_0)$ to generate the next candidate sample $x'_0 \sim q(x'_0|x_0)$ given the current sample $x_0$. After generating the proposal candidate, we decide whether or not to accept this proposal with probability $A(x_0'|x_0)$ defined as
\begin{align*}
    A(x'_0|x_0) = \min \left( 1, \alpha \right), \quad \alpha=\frac{p_\beta(x'_0)q(x_0|x'_0)}{p_\beta(x_0)q(x'_0|x_0)}.
\end{align*}
With a slight abuse of notation, we use $A(x'_0|x_0)$ interchangeably with $A$ when the choice of $x_0'$ and $x_0$ are not important. The MH algorithm can be shown to construct a Markov chain whose stationary distribution is the target distribution. For more details and proof of convergence on the MH algorithm, please refer to App.~\ref{app:mcmc-proofs}. To generate samples from the target reward-weighted distribution $p_\beta(x_0) \propto \exp(r(x_0)/\beta)p(x_0)$, the acceptance probability is as follows:
\begin{equation}
    A (x_0'|x_0) = \min(1,\alpha),
\end{equation}
\begin{equation}
    \alpha = \exp \left( \frac{r(x_0') - r(x_0)}{\beta} \right) \frac{p(x_0') q(x_0|x_0')}{p(x_0) q(x_0'|x_0)}.
\end{equation}
The difficulty in applying the MH algorithm to diffusion models is the intractibility of the acceptance probability. Since $p(x_0)$ is intractable, the acceptance probability cannot be efficiently computed. 

\subsection{Forward-Backward Proposal Distribution}
\vspace{-0.5\baselineskip}
To bypass the computation of $p(x_0)$, we define the proposal distribution in such a way that the proposal probabilities $q(x_0|x_0')$ and $q(x_0'|x_0)$ cancel out $p(x_0)$ and $p(x_0')$. We define the proposal distribution $q(x_0'|x_0)$ using a forward-backward diffusion process as follows. Starting from the current clean sample $x_0$, we choose a random time $t~\sim \gU(t_\text{lo}, t_\text{hi})$ (where $t_\textit{lo}$ and $t_\textit{hi}$ are user-defined parameters) and apply the forward process to obtain a noisy auxiliary sample $x_{t} \sim p_{t}(\cdot | x_0)$. Then, we run the reverse process to obtain a new clean sample $x_0' \sim p(\cdot | x_t)$ This proposal distribution is defined by
\begin{align}
    \label{eq:proposal}
    q(x_0'|x_0) := p_t(x_t|x_0) p(x_0'|x_t),
\end{align}
where $t \sim \gU(t_\text{lo}, t_\text{hi})$ and $x_t \sim p(\cdot | x_0)$ are resampled at each step. 

However, we find that running a one-step reverse process results in an inaccurate $x_0'$ in practice. Instead, we run an $M$-step reverse process on $x_t$ to obtain $(x_{t_M}, x_{t_{M-1}}, \dots, x_1, x_0') \sim \prod_{i=1}^M p(x_{t_{i-1}} | x_{t_{i}})$ where $t=t_{M} > \cdots > t_{0}=0$. We then discard the $x_{t_{i}}$ and retain only $x_0'$. This proposal distribution can be interpreted as exploring locally around $x_0$.

Assuming that the marginal distributions of the learned reverse process matches that of the forward process, we have
\begin{align*}
    \alpha
    &= \exp \left( \frac{r(x_0') - r(x_0)}{\beta} \right) \frac{p(x_0') p(x_t|x_0') p(x_0|x_t)}{p(x_0) p(x_t|x_0) p(x_0'|x_t)} \\ 
    &= \exp \left( \frac{r(x_0') - r(x_0)}{\beta} \right) \frac{p(x_0') p(x_t|x_0') p(x_0|x_t, x_0')}{p(x_0) p(x_t|x_0) p(x_0'|x_t, x_0)} \\
    &= \exp \left( \frac{r(x_0') - r(x_0)}{\beta} \right).
\end{align*}

Thus, the acceptance probability simplifies to
\begin{equation}
\label{eq:mh-accept-probs}
    A(x_0'|x_0)=\min(1,\alpha), \quad \alpha = \exp \left( \frac{r(x_0') - r(x_0)}{\beta} \right),
\end{equation}
which can now be efficiently computed. The assumption that the marginals of the reverse and forward process match is mild and is enforced by the training objective of discrete diffusion models~\citep{austin2021d3pm}. The detailed derivation of the acceptance probability can be found in App.~\ref{app:accept-probs}. 

\ourmethod{} begins by running the full reverse process to obtain a clean initial sample $x_0^{(1)}$. At each step, we generate a candidate sample using our forward-backward diffusion process and choose to accept or reject the candidate based on the MH acceptance probability given in Eq.~\ref{eq:mh-accept-probs}, as shown in Fig.~\ref{fig:method} (right). This results in a chain of clean samples whose stationary distribution is the reward-weighted distribution. While MH is guaranteed to construct a Markov chain whose stationary distribution is $p_\beta(x_0)$, initial samples may not come from the stationary distribution. Thus, we discard the first half of the chain and take as many samples as desired from the latter half at equal intervals. The practice of ``throwing away'' the initial samples of a Markov chain is known as \newterm{burn-in}~\citep{robert2009mh, murphy2023probabilistic}. The full algorithm for \ourmethod \ is shown in Alg.~\ref{alg:mhdiff}. 

\begin{figure}
    \centering
    \input{Algorithm/method-algo}
\end{figure}
\vspace{-\baselineskip}
\section{Experiments}
\vspace{-0.5\baselineskip}
We conduct experiments on molecule generation with QM9~\citep{ramakrishnan2014qm9} and ZINC250K~\citep{irwin2012zinc}, and biological sequence design using the MPRA~\citep{gosai2023mpra} dataset. 
We pretrain four discrete diffusion models on each of the dataset:
\begin{itemize}[leftmargin=*]
    \item \textbf{MDM~\citep{sahoo2024mdlm}:} A masked discrete diffusion model using discrete-time ancestral sampling.
    \item \textbf{USM~\citep{schiff2024udlm}:} A uniform state discrete diffusion model using discrete-time ancestral sampling.
    \item \textbf{SEDD-M~\citep{lou2024sedd}}: A score-based CTMC masked discrete diffusion model.
    \item \textbf{SEDD-U~\citep{lou2024sedd}:} A score-based CTMC uniform discrete diffusion model.
\end{itemize}

We compare the following training-free reward-guided sampling methods for discrete diffusion:

\begin{itemize}[leftmargin=*]
    \item \textbf{Pretrained models:} Samples are generated using the pretrained model.
    \item \textbf{Best-of-N (BoN):} $N$ samples are generated and the one with the highest reward is selected.
    \item \textbf{SMC~\citep{doucet2001smcintro}:} The representative derivative-free particle-based sampling method which approximates the target distribution by updating and resampling a set of $N$ particles at each step using the intermediate reward.
    \item \textbf{SVDD~\citep{li2024svdd}:} $N$ candidate samples are generated at each step and a sample is selected by randomly choosing a sample with probability proportional to its intermediate reward.
    \item \textbf{SGDD~\citep{chu2025sgdd}:} A posterior sampling method specifically tailored for uniform CTMC discrete diffusion models which uses split Gibbs sampling~\citep{vono2019splitgibbs}. Since this method is based on the forward process of uniform CTMC discrete diffusion models, we only evaluate SGDD on SEDD-U and not on the discrete-time model USM.
    \item \textbf{\ourmethod (Ours):} We run Alg.~\ref{alg:mhdiff} to construct a Markov chain converging to $p_\beta(x_0)$ and draw multiple samples from the chain.
    \item \textbf{\ourmethod-B (Ours):} \ourmethod \ with batching by running $B$ chains in parallel while keeping the total NFE fixed by reducing the number of iterations in each Markov chain, resulting in faster sampling times. 
\end{itemize}
\vspace{-0.5\baselineskip}
We match the total diffusion model NFE for each method. More experiment details can be found in App.~\ref{app:experiment-details}.

\noindent
\begin{figure}[t]
\hfill
\begin{minipage}{0.5\textwidth}
\vspace{0pt}
    \centering
    \includegraphics[width=\linewidth]{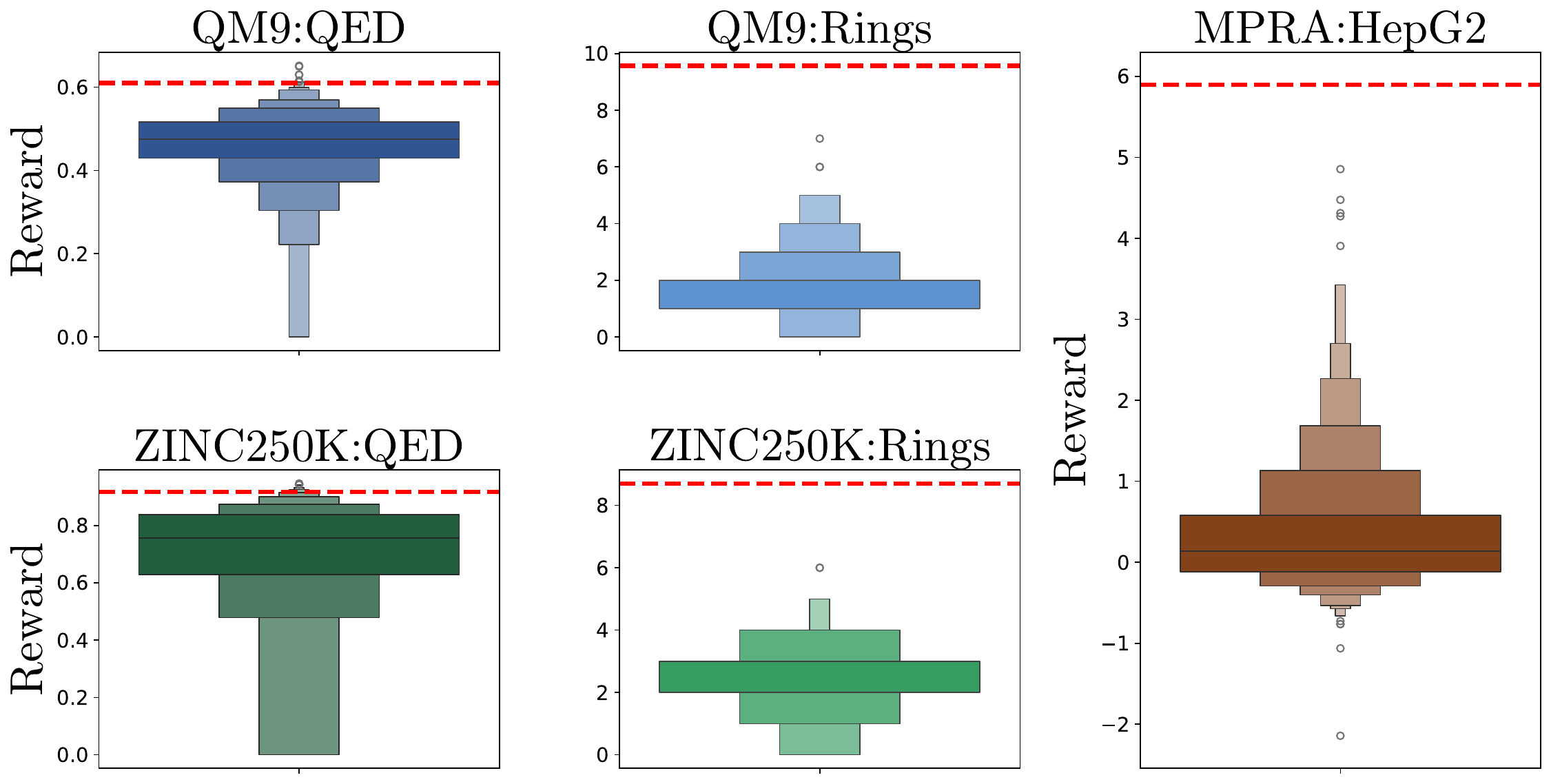}
    \vspace{-\baselineskip}
    \captionof{figure}{\textbf{Reward distributions of the pretrained USM.} The red dotted line represents the average reward achieved by~\ourmethod{}. For ring count and HepG2, the pretrained model reward distribution has low density at higher rewards, resulting in degraded performance for BoN sampling.}
    \label{fig:reward-distribution}
\end{minipage}
\vspace{-\baselineskip}
\end{figure}

\vspace{-1.5\baselineskip}
\paragraph{Molecule Generation.}
We test our method on two molecule datasets: QM9~\citep{ramakrishnan2014qm9} and ZINC250K~\citep{irwin2012zinc}. QM9 is a dataset consisting of $\sim$133,000 small organic molecules, and ZINC250K is a dataset of 250,000 commercially available compounds. The molecules in both datasets are represented as SMILES strings~\citep{weininger1988smiles}. For our reward functions, we use \textbf{QED}~\citep{bickerton2012qed}, ring count \textbf{(Rings)}, and synthetic accessibility \textbf{(SA)}~\citep{ertl2009sa}. QED measures the drug-likeness of a compound based on eight widely used molecular properties (number of hydrogen bond donors/acceptors, molecular polar surface area, number of aromatic rings, etc.). Higher QED values indicate higher drug-likeness. Ring count measures the number of rings in the symmetrized smallest set of smallest rings (SSSR). Finally, SA measures the ease of synthesis of drug-like molecules through a combination of known common structural features in known synthesized molecules, and a penalty based on complex structural features of the molecule. SA takes on a value bewteen 1 and 10 where higher values indicate that the molecule is harder to synthesize. In this work, SA is converted to a reward function by applying the renormalization $(10-SA)/9$ so that higher values indicate better performance. For all rewards, higher is better.

\begin{table*}[h]
\scriptsize
\begin{tabularx}{\linewidth}{>{\centering\arraybackslash}m{0.06\textwidth}|>{\centering\arraybackslash}m{0.08\textwidth}|*{3}{Y}|*{3}{Y}|Y}
\toprule
\multicolumn{2}{c|}{} & \multicolumn{3}{|c|}{\textbf{QM9}} & \multicolumn{3}{|c|}{\textbf{ZINC250K}} & \multicolumn{1}{|c}{\textbf{MPRA}} \\
\multicolumn{2}{c|}{} & \multicolumn{1}{|c}{QED} & \multicolumn{1}{c}{Rings} & \multicolumn{1}{c|}{SA} & \multicolumn{1}{|c}{QED} & \multicolumn{1}{c}{Rings} & \multicolumn{1}{c|}{SA} & \multicolumn{1}{|c}{HepG2} \\ 
\midrule
\midrule
\multirow{6}{*}{MDM}          & Pretrained 
                              & 0.461{\tiny±0.138}                 & 2.755{\tiny±3.432}                          & 0.558{\tiny±.0227}                  
                              & 0.663{\tiny±0.323}                 & 2.020{\tiny±2.488}                          & 0.742{\tiny±0.323}                  
                              & 0.442{\tiny±1.767}                              \\
                              & BoN        
                              & 0.580{\tiny±0.073}             & 5.479{\tiny±1.286}                          & 0.818{\tiny±0.129}            
                              & 0.854{\tiny±0.121}                   & 3.854{\tiny±1.923}                    & 0.862{\tiny±0.121}            
                              & 1.842{\tiny±2.093}                              \\
                              & SMC        
                              & 0.512{\tiny±0.144}                   & 2.803{\tiny±1.713}                          & 0.759{\tiny±0.271}
                              & 0.637{\tiny±0.310}                        & 2.128{\tiny±2.594}                          & 0.769{\tiny±0.310}
                              & 3.087{\tiny±2.975}                        \\
                              & SVDD       
                              & 0.567{\tiny±0.117}                   & 2.849{\tiny±1.462}                          & 0.836{\tiny±0.174}
                              & 0.776{\tiny±0.255}                        & 2.490{\tiny±1.848}                          & 0.754{\tiny±0.255}
                              & 2.319{\tiny±2.654}                              \\
                              & \CC \textbf{\ourmethod}         
                              & \CC \textbf{0.610}{\tiny±0.085}          & \CC \textbf{10.00}{\tiny±0.933}                & \CC \textbf{0.913}{\tiny±0.077}
                              & \CC {\ul 0.910}{\tiny±0.060}               & \CC \textbf{8.091}{\tiny±2.328}                 & \CC \textbf{0.905}{\tiny±0.060}
                              & \CC \textbf{5.259}{\tiny±0.849}                     \\
                              & \CC \textbf{\ourmethod-B}         
                              & \CC \textbf{0.610}{\tiny±0.083}          & \CC {\ul 8.678}{\tiny±2.531}                & \CC {\ul 0.911}{\tiny±0.095}         
                              & \CC \textbf{0.914}{\tiny±0.049}          & \CC {\ul 6.032}{\tiny±1.918}                & \CC {\ul 0.903}{\tiny±0.073}          
                              & \CC {\ul 5.127}{\tiny±1.428}                     \\
\midrule
\multirow{6}{*}{USM}          & Pretrained 
                              & 0.461{\tiny±0.152}                   & 2.032{\tiny±1.341}                          & 0.620{\tiny±0.223}
                              & 0.739{\tiny±0.270}                        & 2.598{\tiny±1.883}                          & 0.768{\tiny±0.189}
                              & 0.351{\tiny±1.541}                              \\
                              & BoN        
                              & 0.600{\tiny±0.054}             & 5.044{\tiny±1.267}                    & 0.839{\tiny±0.097}            
                              & 0.901{\tiny±0.065}                  & 4.184{\tiny±1.307}                    & 0.889{\tiny±0.048}            
                              & 1.753{\tiny±1.965}                        \\
                              & SMC        
                              & 0.485{\tiny±0.170}                   & 1.957{\tiny±1.212}                          & 0.660{\tiny±0.203}
                              & 0.750{\tiny±0.270}                        & 2.633{\tiny±1.828}                          & 0.785{\tiny±0.184}
                              & 0.743{\tiny±2.048}                              \\
                              & SVDD       
                              & 0.454{\tiny±0.157}                   & 2.135{\tiny±1.460}                          & 0.624{\tiny±0.217}
                              & 0.750{\tiny±0.260}                        & 2.562{\tiny±1.861}                          & 0.773{\tiny±0.195}
                              & 0.695{\tiny±1.920}                              \\
                              & \CC \textbf{\ourmethod}         
                              & \CC \textbf{0.610}{\tiny±0.085}          & \CC \textbf{9.570}{\tiny±0.791}                & \CC \textbf{0.908}{\tiny±.063}        
                              & \CC \textbf{0.917}{\tiny±0.050}               & \CC \textbf{8.703}{\tiny±3.454}                 & \CC \textbf{0.898}{\tiny±0.060}
                              & \CC \textbf{5.897}{\tiny±1.499}                     \\
                              & \CC \textbf{\ourmethod-B}         
                              & \CC {\ul 0.609}{\tiny±0.081}             & \CC {\ul 9.240}{\tiny±2.401}                 & \CC {\ul 0.915}{\tiny±0..088}        
                              & \CC {\ul 0.910}{\tiny±0.063}         & \CC {\ul 6.310}{\tiny±2.494}                 & \CC {\ul 0.895}{\tiny±0.073}         
                              & \CC {\ul 5.292}{\tiny±1.876}                    \\
\midrule
\multirow{6}{*}{SEDD-M}    & Pretrained 
                           & 0.460{\tiny±0.143}                   & 2.305{\tiny±3.318}                          & 0.588{\tiny±0.240}                  
                           & 0.669{\tiny±0.328}                        & 2.350{\tiny±2.648}                          & 0.731{\tiny±0.243}
                           & 0.373{\tiny±1.631}                              \\
                           & BoN        
                           & 0.582{\tiny±0.069}             & 5.241{\tiny±2.441}                    & {\ul 0.831}{\tiny±0.123}            
                           & {\ul 0.852}{\tiny±0.119}                  & 4.057{\tiny±2.022}                    & 0.857{\tiny±0.088}            
                           & 1.874{\tiny±2.243}                        \\
                           & SMC        
                           & 0.461{\tiny±0.162}         & 2.425{\tiny±3.365}         & 0.599{\tiny±0.263} 
                           & 0.669{\tiny±0.338}              & 2.251{\tiny±2.856}         & 0.745{\tiny±0.224} 
                           & 0.386{\tiny±1.639} \\
                           & SVDD       
                           & 0.450{\tiny±0.159}         & 2.378{\tiny±3.243}         & 0.575{\tiny±0.244} 
                           & 0.666{\tiny±0.326}              & 2.362{\tiny±2.707}         & 0.739{\tiny±0.245} 
                           & 0.456{\tiny±1.824} \\
                           & \CC \textbf{\ourmethod}         
                           & \CC \textbf{0.619}{\tiny±0.096}          & \CC \textbf{9.792}{\tiny±3.402}                 & \CC \textbf{0.894}{\tiny±0.113}         
                           & \CC \textbf{0.875}{\tiny±0.150}               & \CC \textbf{10.026}{\tiny±4.305}                & \CC \textbf{0.885}{\tiny±0.088}
                           & \CC \textbf{7.153}{\tiny±1.439}                     \\
                           & \CC \textbf{\ourmethod-B} 
                           & \CC {\ul 0.594}{\tiny±0.098}         & \CC {\ul 8.977}{\tiny±3.041}      & \CC 0.821{\tiny±0.262}
                           & \CC 0.846{\tiny±0.189}                           & \CC {\ul 5.677}{\tiny±5.179}      & \CC {\ul 0.863}{\tiny±0.162}
                           & \CC {\ul 5.991}{\tiny±1.643}                       \\
\midrule
\multirow{7}{*}{SEDD-U} & Pretrained 
                        & 0.458{\tiny±0.154}                   & 1.654{\tiny±2.355}                          & 0.637{\tiny±0.230}                  
                        & 0.741{\tiny±0.266}                        & 2.524{\tiny±1.753}                          & 0.771{\tiny±0.189}
                        & 0.455{\tiny±1.755}                              \\
                        & BoN        
                        & {\ul 0.583}{\tiny±0.066}             & 3.862{\tiny±1.996}                    & 0.843{\tiny±0.112}            
                        & {\ul 0.904}{\tiny±0.045}                  & {\ul 4.056}{\tiny±1.241}                    & {\ul 0.889}{\tiny±0.049}
                        & 1.729{\tiny±2.216}                              \\
                        & SMC        
                        & 0.546{\tiny±0.113}                   & 2.718{\tiny±2.790}                          & 0.811{\tiny±0.196}                  
                        & 0.850{\tiny±0.174}                        & 2.610{\tiny±1.684}                          & 0.847{\tiny±0.136}
                        & 3.334{\tiny±3.337}                                \\
                        & SVDD       
                        & 0.518{\tiny±0.120}                   & 2.482{\tiny±2.624}                          & 0.779{\tiny±0.234}                  
                        & 0.809{\tiny±0.217}                   & 2.759{\tiny±1.784}                          & 0.817{\tiny±0.148}                  
                        & 3.189{\tiny±3.330}                                \\
                        & SGDD       
                        & 0.536{\tiny±0.121}                   & 2.644{\tiny±2.791}                          & 0.684{\tiny±0.211}                  
                        & 0.844{\tiny±0.152}                   & 2.535{\tiny±1.627}                          & 0.847{\tiny±115}                  
                        & 9.240{\tiny±2.050}                        \\
                        & \CC \textbf{\ourmethod}         
                        & \CC \textbf{0.619}{\tiny±0.080}      & \CC {\bf 7.488}{\tiny±5.376}                & \CC {\ul 0.886}{\tiny±0.075}         
                        & \CC \textbf{0.922}{\tiny±0.073}      & \CC \textbf{5.499}{\tiny±2.614}             & \CC \textbf{0.898}{\tiny±0.074}         
                        & \CC \textbf{10.09}{\tiny±2.637} \\
                        & \CC \textbf{\ourmethod-B} 
                        & \CC 0.572{\tiny±0.068}               & \CC {\ul 7.213}{\tiny±4.325}         & \CC \textbf{0.908}{\tiny±0.120}
                        & \CC 0.894{\tiny±0.112}               & \CC 3.625{\tiny±0.870}               & \CC 0.883{\tiny±0.088}
                        & \CC {\ul 9.350}{\tiny±1.762} \\
\bottomrule
\end{tabularx}
\caption{\textbf{Average reward (with 95\% confidence intervals) for each method and pretrained discrete diffusion model.} Higher is better. \textbf{Bold} indicates the best method, and {\ul underline} denotes the second best. \ourmethod{} achieves the best average reward across all datasets and discrete diffusion models.}
\label{tab:main-results}
\vspace{-\baselineskip}
\end{table*}

\vspace{-0.5\baselineskip}
\paragraph{Biological Sequence Generation.}
We pretrain the discrete diffusion models on the DNA dataset provided by \citet{gosai2023mpra} (which we refer to as MPRA) which measures the enhancer acitivity of $\sim$700,000 DNA sequences using massively parallel reporter assays (MPRA). An Enformer model~\citep{avsec2021enformer} is trained to predict the enhancer activity level in the HepG2 cell line which is used as the reward function. We use the model trained on two different subsets of data and use one exclusively to provide the guidance signal during sampling and the other for evaluation. Higher predicted HepG2 activity indicates better performance.

\vspace{-0.5\baselineskip}
\paragraph{Results.} The quantitative results of reward-guided generation are shown in Tab.~\ref{tab:main-results}. As shown, \ourmethod \ or \ourmethod-B achieves the best reward in all cases. We also compute the diversity in Tab.~\ref{tab:diversity} in App.~\ref{subapp:diversity}, showing that the samples remain sufficienctly diverse despite achieving high reward.

SMC and SVDD require intermediate rewards which are often inaccurate. As such, these methods sometimes only yield slightly improved results over the pretrained model and, in some cases, slightly worse results due to inaccurate exploration. The performance of these methods is also highly reliant on both the smoothness of the reward function and the diffusion model's accuracy in $x_0$ prediction. \ourmethod, on the other hand, only use accurate clean rewards for guidance, resulting in accurate and efficient guidance. 

While BoN performs well on QED and SA for QM9 and ZINC250K, its performance is not as good on the HepG2 activity reward for MPRA and the ring count reward for QM9 and ZINC250K. This is due to the lack of trajectory guidance: BoN performs well only for settings where high-reward samples are likely to be sampled by the pretrained model and performs worse in cases where high-reward samples lie in low-density regions of the pretrained model. As shown in Fig.~\ref{fig:reward-distribution}, for ring count and HepG2 rewards, the high-reward samples generated by \ourmethod \ (red dashed line) lie in low-density regions unlikely to be sampled from by the pretrained model. \ourmethod, which uses the forward-backward diffusion process to guide the samples toward high-reward regions, is able to achieve high rewards across all datasets and reward functions even when high-reward samples lie in low-density regions of the pretrained model.

Although SGDD achieves the second-best result on the MPRA dataset, it is only applicable to uniform CTMC discrete diffusion models whereas our \ourmethod \ is applicable to all discrete diffusion models. Furthermore, SGDD still relies on computing intermediate rewards, limiting its performance on datasets such as QM9 and ZINC250K where the intermediate rewards are inaccurate.
\vspace{-\baselineskip}
\paragraph{Wall-Clock Time.} We include comparisons of wall-clock time and performance in Tab.~\ref{tab:zinc250k-matched-wall-clock}. Under matched wall-clock time, \ourmethod-B significantly outperforms all other baselines in ring count and HepG2. Additional wall-clock time comparisons are included in App.~\ref{app:wall-clock}.
\subsection{Comparison with Training-based Methods}
We provide comparison with standard training-based guidance methods. In Tab.~\ref{tab:training-based-comparison}, we compare \ourmethod{} against discrete classifier-free guidance (D-CFG)~\citep{nisonoffdiscreteguidance} on the QM9 dataset. \ourmethod{} outperforms D-CFG across all rewards when using MDM as the base model. When using USM as the base model, \ourmethod{} outperforms D-CFG on QED and ring count while achieving comparable SA. Furthermore, \ourmethod{} can be applied on top of D-CFG, resulting in further improvements except in ring count for ZINC250K MDM. This is due to the MDM CFG model collapsing to narrow modes when trained on ZINC250K ring count, resulting in insufficient exploration or the space when applying \ourmethod.

\begin{table}[ht]
{
\scriptsize
\begin{tabularx}{\linewidth}{>{\centering\arraybackslash}m{0.1\textwidth}|>{\centering\arraybackslash}m{0.07\textwidth}YYYY}
             \toprule
              & \multicolumn{5}{c}{\textbf{ZINC250K: Rings}} \\
             & \ourmethod-B  & BoN   & BoN   & SMC   & SVDD \\
             \midrule
Batch Size   & 8     & 78    & 8     & 14    & 886 \\
NFE         & 1024  & 5750  & 1024  & 1024  & 65536 \\
\rowcolor{gray!12} Reward       & 5.315 & 4.945 & 4.312 & 2.630 & 2.607 \\
\rowcolor{gray!12} Time (s)     & 321   & 331   & 359   & 320   & 276  \\
\bottomrule
\end{tabularx}
}
\caption{\textbf{ZINC250K rewards under matched wall-clock time.} Rewards and time are computed from 128 drawn samples. \ourmethod{}-B significantly outperforms all other baselines in ring count while being comparable with BoN in QED and SA.}
\vspace{-\baselineskip}
\label{tab:zinc250k-matched-wall-clock}
\end{table}

\begin{table}[h]
\scriptsize
\begin{tabularx}{\linewidth}{>{\centering\arraybackslash}m{0.06\textwidth}|>{\centering\arraybackslash}m{0.1\textwidth}|*{3}{Y}}
\toprule
\multicolumn{2}{c|}{} & \multicolumn{3}{c}{\bf QM9} \\
\multicolumn{2}{c|}{} & \multicolumn{1}{Y}{QED} & \multicolumn{1}{Y}{Rings} & \multicolumn{1}{Y}{SA} \\ 
\midrule
\multirow{3}{*}{MDM}          & D-CFG & 0.820 & 5.886 & 0.888   \\
                              
                              & \ourmethod{} & 0.910 & {\bf 8.091} & 0.905 \\
                              
                              & D-CFG + \ourmethod{} & {\bf 0.927} & 7.214 & {\bf 0.928} \\
                              \midrule
\multirow{3}{*}{USM}          & D-CFG & 0.903 & 5.613 & 0.908   \\
                              
                              & \ourmethod{} & 0.917 & 8.703 & 0.898 \\
                              
                              & D-CFG + \ourmethod{} & {\bf 0.927} & {\bf 8.943} & {\bf 0.920} \\
\bottomrule
\end{tabularx}
\caption{\textbf{Comparison with training-based guidance on ZINC250K.} \ourmethod{} outperforms D-CFG. \ourmethod{} can also be combined with D-CFG by adopting D-CFG-based proposals, resulting in further improvements.}
\vspace{-\baselineskip}
\label{tab:training-based-comparison}
\end{table}

\subsection{Analysis and Hyperparameter Studies} In this section, we analyze the effects of the various components of \ourmethod. 
\vspace{-\baselineskip}
\paragraph{Convergence.} Since \ourmethod \ requires running sufficiently many iterations until the Markov chain converges to its stationary distribution, we empirically analyze the speed of convergence as well as the impact of the chain length. The autocorrelation function shown in Fig.~\ref{fig:acf-zinc} quickly decays to zero within the first 2000 iterations and remains small until the end, indicating that the chain converges quickly. We also report additional autocorrelation plots as well as the acceptance rate in App.~\ref{subapp:acceptance-rate}.

\begin{figure}[h]
    \centering
    \includegraphics[width=0.9\linewidth]{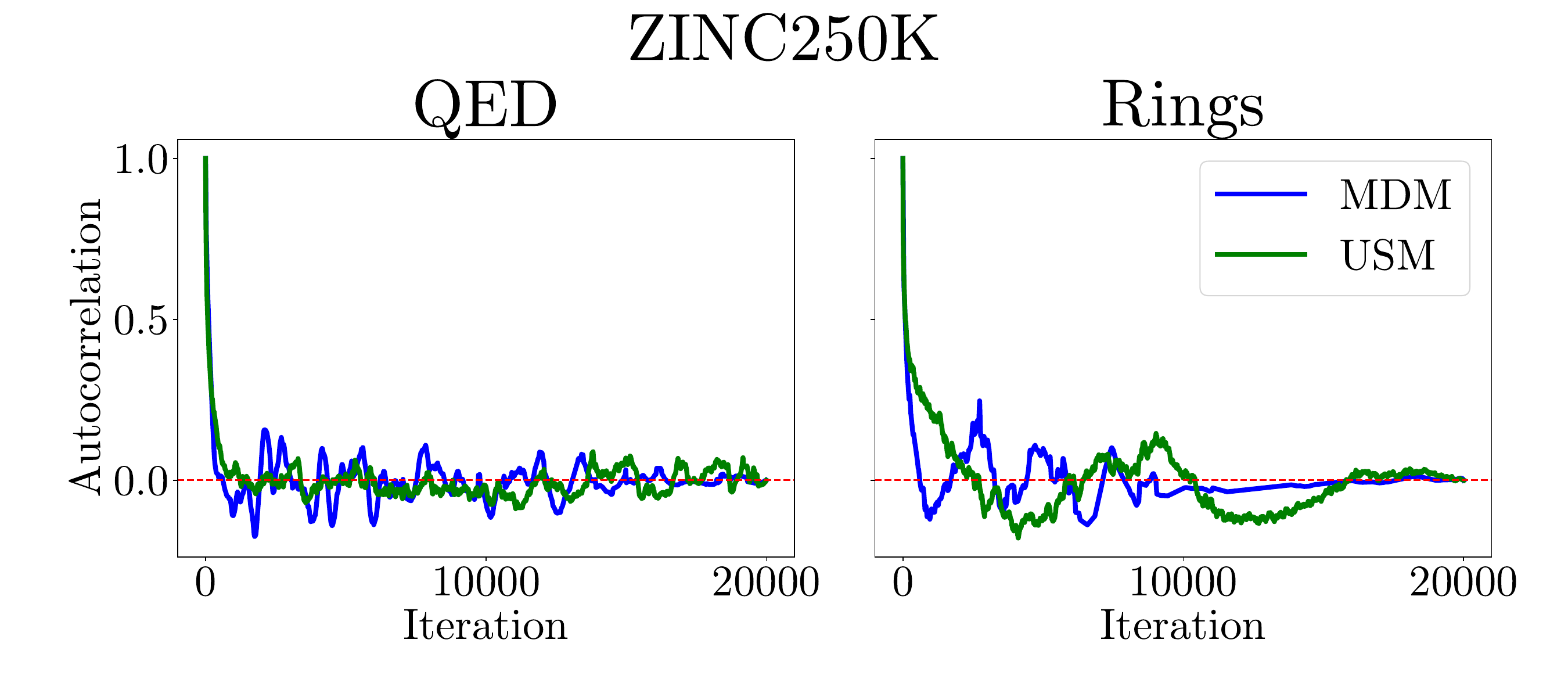}
    \vspace{-\baselineskip}
    \caption{\textbf{Autocorrelation plots for ZINC250K MDM and USM.} The autocorrelation function quickly vanishes to zero within the first 2000 iterations, indicating fast mixing.}
    \label{fig:acf-zinc}
\end{figure}
\vspace{-\baselineskip}
\paragraph{NFE.} We also provide additional experiments varying the NFE in App.~\ref{app:nfe}, showing that \ourmethod{} outperforms the baseline across various NFEs.
\vspace{-0.5\baselineskip}
\paragraph{Initialization.} Since \ourmethod{} constructs a Markov chain of clean samples, the choice of initialization could impact the convergence rate of the chain. We have already shown in Fig.~\ref{fig:acf-zinc} that the Markov chain shows fast mixing. We now provide further evidence that the initialization $x_0^{(1)}$ of the chain does not significantly affect the convergence. In Fig.~\ref{fig:warm-start}, we plot the reward trajectories for \ourmethod{} with 64 different initializations. The reward distribution across the 64 chains are shown as histograms on top (corresponding to the iterations marked by the gray dashed lines). Although each initial point starts with a different reward, many of which are zero, as shown in the histogram, the rewards of each chain quickly converge to the same high-reward region.

\begin{figure}
    \centering
    \includegraphics[width=\linewidth]{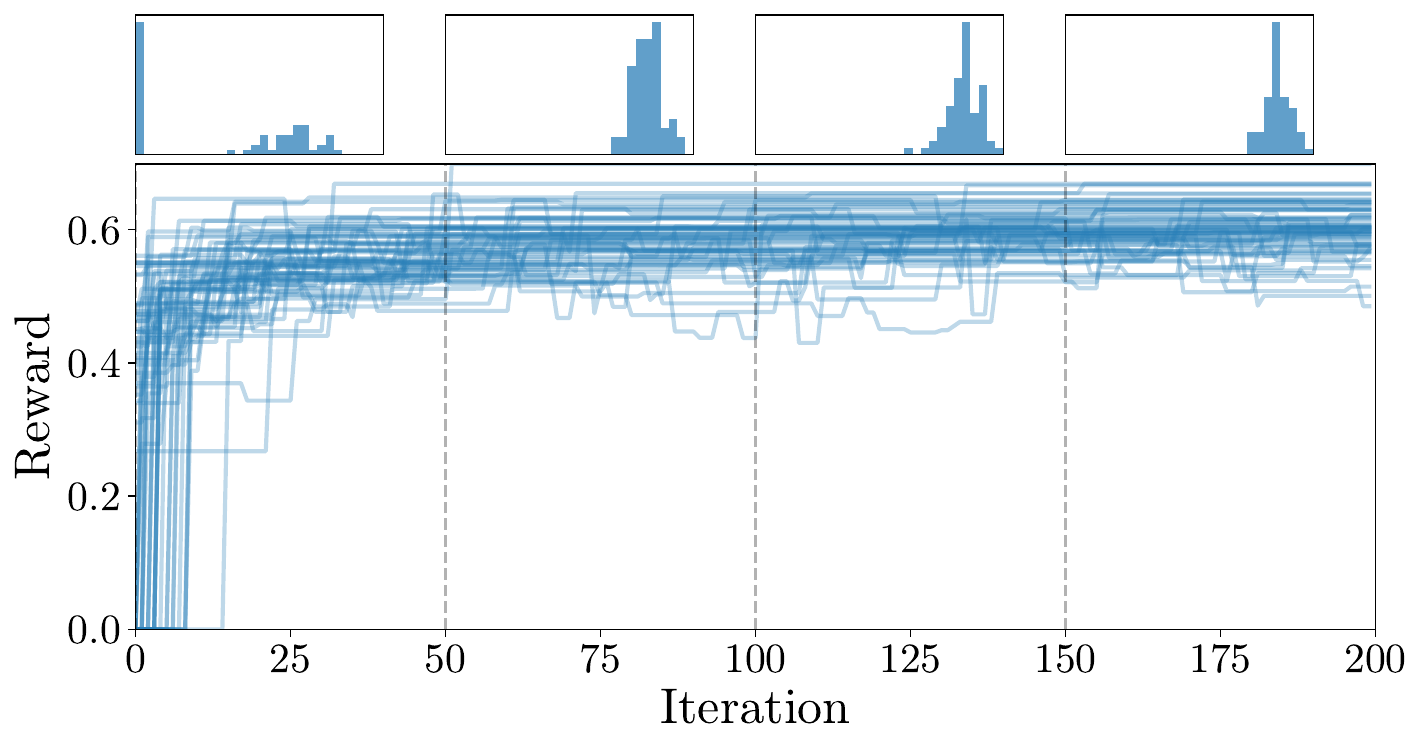}
    \vspace{-1.5\baselineskip}
    \caption{\textbf{Reward trajectories with different initializations.} We plot the reward trajectories of \ourmethod{} for 64 different initializations using MDM on QM9 QED (bottom), along with the reward distribution of the 64 different chains (top). Trajectories with lower reward also quickly converge to high reward regions of the distribution, demonstrating robustness to the initialization.}
    \label{fig:warm-start}
\end{figure}
\vspace{-\baselineskip}
\paragraph{Results with Varying Number of Reverse Steps $M$.} \ourmethod \ enables fast sampling by leveraging an $M$-step reverse sampler in the forward-reverse proposal to produce a new candidate sample. As shown in Tab.~\ref{tab:sdedit-steps}, varying $M$ does not significantly affect the rewards achieved by \ourmethod, except when using large $M$ for the ring count reward. We hypothesize that this is due to limited compute since using for ZINC250K tasks results in only 88 iterations of \ourmethod, which may not be sufficient for harder tasks such as ring count where high reward samples lie in low density regions of the pretrained model's distribution. This hypothesis is also supported by the autocorrelation plot in Fig.~\ref{fig:acf-zinc} which shows that the autocorrelation function for USM (green) requires more iterations to converge for ring count.

\begin{table}[!h]
{
\small
\begin{tabularx}{\linewidth}{Y|YYY}
             \toprule
             & \multicolumn{3}{c}{\textbf{ZINC250K}} \\
             M & QED  & Ring   & SA \\
             \midrule
5 & 0.917 & 8.703 & 0.898 \\
10 & 0.910 & 8.477 & 0.901 \\
15 & 0.917 & 8.256 & 0.895 \\
20 & 0.917 & 6.885 & 0.897 \\
\bottomrule
\end{tabularx}
}
\caption{\textbf{Results with varying number of reverse steps $M$ for ZINC250K USM.} We vary the number of reverse steps $M$ in the forward-reverse proposal. Varying does not significantly affect the rewards achieved by \ourmethod.}
\vspace{-\baselineskip}
\label{tab:sdedit-steps}
\end{table}

\paragraph{Time Parameters $t_\text{lo}$ and $t_\text{hi}$.} We plot the reward while varying the time parameters $t_\text{lo}$ and $t_\text{hi}$ in Fig.~\ref{fig:t-ablation} in App.~\ref{subapp:time-params}. As shown, the reward and diversity achieved by \ourmethod{} is robust to the choice $t_\text{lo}$ and $t_\text{hi}$.
\vspace{-0.5\baselineskip}
\section{Conclusion}
\vspace{-0.5\baselineskip}
We propose \textbf{\ourmethod}, a novel training-free reward-guided sampler for discrete diffusion which avoids reliance on noisy intermediate rewards based on constructing a Markov Chain of clean samples via the Metropolis-Hastings algorithm. Our proposal distribution, modeled through a forward–backward combination, makes the acceptance probability tractable. Experiments on molecular and biological sequence generation with various reward functions demonstrate superior performance of our method compared to previous methods that rely on intermediate rewards.

\vspace{-0.5\baselineskip}
\paragraph{Limitations and Future Work.} While the use of the MH algorithm allows \ourmethod{} to sample from the target distribution, theoretical convergence guarantees depend on the existence of a perfect denoiser. This assumption is mild since discrete diffusion models are trained to match the marginal distributions of the forward process~\citep{austin2021d3pm}. Exact quantification of the discrepancy in the learned marginal distributions is an open area of research and is left for future works. 

\bibliography{references}
\bibliographystyle{icml2026}

\newpage
\appendix
\onecolumn
\section*{Appendix}
\section{Metropolis-Hastings Algorithm}
\label{app:mcmc-proofs}
In this section, we present a brief overview of the Metropolis-Hastings (MH) algorithm. For more details, we refer the reader to chapter 12 section 2 of \citet{murphy2023probabilistic}.

The MH algorithm constructs a Markov chain which converges to a target distribution $p^\star(x)$. In order to do so, a \newterm{proposal distribution} $q(\cdot | x)$ proposes to move from the current state $x$ to a new state $x' \sim q(\cdot | x)$. After proposing the new state $x'$, MH decides whether to accept or reject the new state with the acceptance probability
\begin{equation*}
    A(x'|x) = \min \left( 1, \alpha \right), \quad \alpha=\frac{p^\star(x')q(x|x')}{p^\star(x)q(x'|x)}.
\end{equation*}

Intuitively, the accept-reject step is necessary as the proposal distribution may not match the target distribution. Samples closer to the target distribution are accepted with higher probability whereas samples further from the target distribution are rejected. The MH algorithm is summarized in Alg.~\ref{alg:mh-algo}. The Markov chain constructed by the MH algorithm has the following transition matrix:
\begin{equation}
\label{eq:mh-transition}
    p(x'|x) = 
    \begin{cases}
        q(x'|x) A(x'|x) & \text{if } x' \not= x \\
        q(x|x) + \sum_{x' \not= x} q(x'|x)(1 - A(x'|x)) & \text{otherwise}
    \end{cases}
\end{equation}

\begin{figure}[h]
\begin{algorithm}[H]
\small
\caption{\textbf{Metropolis-Hastings Algorithm (MH)}}
\label{alg:mh-algo}
\begin{algorithmic}[1]
\FOR{$k = 1, \dots, K$}
    \STATE $\tilde{x}^{(k+1)} \sim q(\cdot | x^{(k})$
    \STATE $\alpha \gets \frac{p^\star(x') q(x|x')}{p^\star(x) q(x'|x)}$
    \STATE $A \gets \min (1, \alpha)$
    \STATE $\rho \sim \gU(0,1)$
    \STATE
    $x^{(k+1)} \gets \tilde{x}^{(k+1)} \text{ if } \rho < A \text{ else } x^{(k)}
    $
\ENDFOR \\
\STATE $\{ x^{(1)}, \dots, x^{(K)} \}$
\end{algorithmic}
\end{algorithm}
\end{figure}

\begin{theorem}[Theorem 12.2.1 from~\cite{murphy2023probabilistic}]
    If the transition matrix defined by Eq.~\ref{eq:mh-transition} defined by the MH algorithm is ergodic and irreducible, $p^\star$ is its unique limiting distribution.
\end{theorem}
\begin{proof}
    Consider two states $x'$ and $x$. A Markov chain is said to satisfy the detailed balance equation if the following holds:
    \begin{equation}
    \label{eq:detailed-balance}
        p^\star(x)p(x'|x) = p^\star(x')p(x|x').
    \end{equation}

    It is known that if a Markov chain satisfies the detailed balance equation, then $p^\star$ is its stationary distribution (Theorem 2.6.3 from~\citet{murphy2023probabilistic}).
    
    To show that $p^\star$ is the unique limiting distribution of the Markov chain defined by Eq.~\ref{eq:mh-transition}, it suffices to show that it satisfies the detailed balance condition in Eq.~\ref{eq:detailed-balance}. 

    Without loss of generality, assume $p^\star(x)q(x'|x) \ge p^\star(x')q(x|x')$. Then, $\alpha = \frac{p^\star(x')q(x|x')}{p^\star(x)q(x'|x)} < 1$ and thus $A(x'|x)=\alpha$. Similarly, by switching the arguments, $A(x|x')=1$. 

    To move from $x$ to $x'$, $x'$ must be proposed and accepted. Hence,
    \begin{align*}
        p(x'|x) &= q(x'|x)A(x'|x) & \quad \text{from Eq.~\ref{eq:mh-transition}} \\
        &= \frac{p^\star(x')q(x|x')}{p^\star(x)}
    \end{align*}

    It suffices to show that $q(x|x')=p(x|x')$:
    \begin{align*}
        q(x|x') &= q(x|x')A(x|x') & \quad \because A(x|x')=1 \\
        &= p(x|x') & \quad \text{from Eq.~\ref{eq:mh-transition}}
    \end{align*}

    Since the MH Markov chain satisfies the detailed balance equation, $p^\star$ is its stationary distribution.
\end{proof}
\section{Derivation of the Acceptance Probability}
\label{app:accept-probs}
Recall from Sec.~\ref{sec:method} that the proposal distribution $q(\cdot | x_0)$ is defined in Eq.~\ref{eq:proposal} as:
\begin{equation}
\label{eq:forward-reverse-process}
    q(x_0'|x_0) := p_t(x_t|x_0) p(x_0'|x_t)
\end{equation}

We calculate $\alpha$ which is used to compute the acceptance probability $A$. First, we draw and fix a noisy sample $x_t \sim p_t(\cdot | x_0)$. Once $x_t$ has been drawn, $\alpha$ simplifies to

\begin{align*}
    \alpha &= \exp \left( \frac{r(x_0') - r(x_0)}{\beta} \right) \frac{p(x_0') q(x_0|x_0')}{p(x_0) q(x_0'|x_0)} \\
    &= \exp \left( \frac{r(x_0') - r(x_0)}{\beta} \right) \frac{p(x_0') p(x_t|x_0') p(x_0|x_t)}{p(x_0) p(x_t|x_0) p(x_0'|x_t)} \\
    &= \exp \left( \frac{r(x_0') - r(x_0)}{\beta} \right) \frac{p(x_0') p(x_t|x_0') p(x_0|x_t, x_0')}{p(x_0) p(x_t|x_0) p(x_0'|x_t, x_0)} \\
    &= \exp \left( \frac{r(x_0') - r(x_0)}{\beta} \right) \frac{p(x_0,x_t,x_0')}{p(x_0',x_t,x_0)} \\
    &= \exp \left( \frac{r(x_0') - r(x_0)}{\beta} \right),
\end{align*}

where the third line follows from the conditional independence of $x_0$ and $x_0'$ given $x_t$ due to the Markov property: $p(x_0|x_t) = p(x_0|x_t,x_0')$. \\

Although it may appear that this Markov chain is time-inhomogeneous since $t \sim \gU(t_\text{lo}, t_\text{hi})$ and $x_t \sim p(\cdot | x_0)$ are resampled at each step, it is secretly \emph{time-homogeneous}. To see this, we simplify the proposal distribution to an equivalent time-homogeneous proposal distribution. The probability of proposing $x_0'$ given that we are currently at state $x_0$ is $p_t(x_t|x_0) p(x_0'|x_t)$ with probability $\frac{1}{t_\text{hi}-t_\text{lo}}$. To obtain the probability $q(x_0'|x_0)$, we simply condition the proposal probability on $t$ and $x_t$ and integrate over all possible $t$ and $x_t$:
\begin{align*}
    q(x_0'|x_0) 
    &= \int_{t_\text{lo}}^{t_\text{hi}} \int_{\mathbb{X}_t} p_t(x_t|x_0)p(x_0'|x_t)f(t) dx_t dt \\
    &= \E_{t \sim \gU(t_\text{lo}, t_\text{hi})} \E_{x_t \sim p_t(\cdot | x_0)} [p(x_0'|x_t)],
\end{align*}
which is time-homogeneous. Thus, we still have an equivalent time-homogeneous Markov chain and standard convergence proofs in App.~\ref{app:mcmc-proofs} apply.
\section{Experiment Details}
\label{app:experiment-details}

\subsection{Pretrained Models}
We provide details on the training setup and hyperparameters of each diffusion model on each dataset in Tab.~\ref{tab:training}. For QM9 and ZINC250K, we use the transformer architecture for the diffusion, whereas for MPRA we use the CNN architecture, following~\citet{starkdirichletfm} and \citet{wangdrakes}. Note that the SEDD-U model pretrained on MPRA is taken from the publicly available checkpoint provided by~\citet{chu2025sgdd}. All models except for SEDD-U was trained using Adam~\citep{adam2014adam} while SEDD-U was trained using AdamW~\citep{loshchilovAdamW}.

\begin{table}[h]
\centering
\renewcommand{\arraystretch}{1.2} 
\small
\begin{tabularx}{\linewidth}{YYYY|YY}
\toprule
& \multicolumn{3}{c|}{MDM, USM, and SEDD-M} & \multicolumn{2}{c}{SEDD-U}\\
& QM9 & ZINC250K & MPRA & QM9 & ZINC250K\\
\midrule
Train steps & 25,000 & 50,000 & 131,500 & 40,000 & 100,000 \\
Context size & 32 & 74 & 200 & 32 & 74\\
Batch size & 1024 & 384 & 512 & 512 & 512\\
LR & $3e^{-4}$ & $3e^{-4}$ & $2e^{-3}$ & $3e^{-4}$ & $3e^{-4}$ \\
Optim. & \textsc{Adam} & \textsc{Adam} & \textsc{Adam} & \textsc{AdamW} & \textsc{AdamW} \\
 & (0.9, 0.999) & (0.9, 0.999) & (0.9, 0.999) & (0.9, 0.999) & (0.9, 0.999) \\
LR sched. & Constant Warmup & Constant Warmup & Cosine Decay & Constant Warmup & Constant Warmup \\
 & - & - & $3e^{-6}$ min. & - & - \\
LR warmup steps & 2,500 & 2,500 & 3,000 & 2,500 & 2,500 \\
\midrule
GPU count & 2 & 2 & 2 & 8 & 4 \\
GPU type & RTX3090 & RTX3090 & RTX3090 & RTX3090 & RTX3090 \\
\bottomrule
\end{tabularx}
\caption{\textbf{Training setup and Hyperparameters.} Training setup and hyperparameters on QM9, ZINC250K, and MPRA using MDM~\citep{sahoo2024mdlm}, USM~\citep{schiff2024udlm}, SEDD-U~\citep{lou2024sedd}, and SEDD-M~\citep{lou2024sedd}.}
\label{tab:training}
\end{table}

\subsection{Reward-Guided Sampling}
We provide experimental details on each evaluation setup. We sample 1024 molecules for QM9 and ZINC250K, and 640 DNA sequences for MPRA. We fix 32 denoising steps for QM9, 74 denoising steps for ZINC250K, and 128 denoising steps for MPRA. For \ourmethod-B, we use a batch size of 8 for QM9 and ZINC250K, and a batch size of 4 for MPRA.

For experiments on the QM9~\citep{ramakrishnan2014qm9} and ZINC250K~\citep{irwin2012zinc} datasets, we fix the total diffusion model NFE per sample to be 1024. For MPRA~\citep{gosai2023mpra}, we fix the NFE as 1000 as done by \citet{chu2025sgdd}. Since one run of \ourmethod \ generates multiple samples, we draw $S$ samples from the resulting chain by discarding the first half of the chain (burn-in) and taking $S$ equally-spaced samples from the latter half. Since $S$ samples are generated, we scale the NFE by $S$ for \ourmethod. This scaling highlights one of the key benefits of \ourmethod: after the initial burn-in period, samples can be generated quickly using a few steps. For all experiments, we fix $\beta=0.02$. To obtain an $x_0$ prediction from SEDD, we run sampling using one step to jump to time $t=1$ (clean sample).

Hyperparameters used by our method is shown in Tab.~\ref{tab:ours-hyperparameters}. Initial denoising steps refer to the number of steps used to obtain the initial clean sample $x_0^{(1)}$ whereas denoising steps $M$ refer to the number of steps used during the $M$-step reverse process in every MH iteration. Note that $M$ is much smaller than the initial denoising steps.

\begin{table}[h]
\centering
\renewcommand{\arraystretch}{1.2} 
\small
\begin{tabularx}{\linewidth}{YYYY}
\toprule
& QM9 & ZINC250K & MPRA \\
\midrule
Samples per Iteration $S$ & 128 & 128 & 64 \\
MH Iterations $K$ & 26208 & 26199 & 6390 \\
Initial Denoising Steps & 32 & 74 & 100 \\
Denoising Steps $M$ & 5 & 5 & 10 \\
$t_\textit{l}$ & 0.2 & 0.2 & 0.2\\
$t_\textit{h}$ & 0.5 & 0.5 & 0.7 \\
\bottomrule
\end{tabularx}
\caption{\textbf{Hyperparameters for \ourmethod.}}
\label{tab:ours-hyperparameters}
\end{table}
\vspace{-\baselineskip}
\section{Additional Experimental Results}
\label{app:diagonistics}
\subsection{Diversity}
\label{subapp:diversity}
We compute the diversity as follows:
\begin{itemize}
    \item For molecule tasks (QM9 and ZINC250K), we compute the mean pairwise Tanimoto similarity based on the Morgan2 fingerprint, and subtract it from 1.
    \item For the DNA task (MPRA), we subtract the mean pairwise cosine similarity of the one-hot encodings from 1.
\end{itemize}
\vspace{-0.5\baselineskip}
The results are shown in Tab.~\ref{tab:diversity}. BoN and \ourmethod{} \ have a slight decrease in diversity while SMC suffers from significant degradation in diversity. This can be attributed to samples clustering around the high-reward modes. \ourmethod{} consistently achieves a diversity score of over 0.8 on molecule tasks which suggests that the samples are indeed diverse, with an average Tanimoto similarity of less than 0.2. On DNA tasks, the cosine similarity between the sequences is less than 0.3, indicating that the generated sequences are diverse.
\vspace{-\baselineskip}
\begin{table}[h]
\small
\begin{tabularx}{\linewidth}{>{\centering\arraybackslash}m{0.1\textwidth}|>{\centering\arraybackslash}m{0.13\textwidth}|*{3}{Y}|*{3}{Y}|Y}
\toprule
\multicolumn{2}{c|}{} & \multicolumn{3}{|c|}{\textbf{QM9}} & \multicolumn{3}{|c|}{\textbf{ZINC250K}} & \multicolumn{1}{|c}{\textbf{MPRA}} \\
\multicolumn{2}{c|}{} & \multicolumn{1}{|c}{QED} & \multicolumn{1}{c}{Rings} & \multicolumn{1}{c|}{SA} & \multicolumn{1}{|c}{QED} & \multicolumn{1}{c}{Rings} & \multicolumn{1}{c|}{SA} & \multicolumn{1}{|c}{HepG2} \\ 
\midrule
\midrule
\multirow{5}{*}{MDM}          & Pretrained & 0.922                   & 0.922                          & 0.922                  & 0.910                   & 0.910                          & 0.910                  & 0.749                              \\
                              & BoN        & 0.904             & 0.884                    & 0.894            & 0.876             & 0.874                    & 0.876           & 0.749                              \\
                              & SMC        & 0.834                   & 0.916                          & 0.802                  & 0.929                   & 0.918                          & 0.876                  & 0.747                        \\
                              & SVDD       & 0.920                   & 0.920                          & 0.920                  & 0.900                   & 0.896                          & 0.900                  & 0.747                              \\
                               & \CC \textbf{\ourmethod}         & \CC 0.885          & \CC 0.804                & \CC 0.862         & \CC 0.874          & \CC 0.797                 & \CC 0.874         & \CC 0.742                     \\
                               & \CC \textbf{\ourmethod-B} 
                               & \CC 0.891 & \CC 0.854 & \CC 0.868 
                               & \CC 0.872 & \CC 0.877 & \CC 0.839 
                               & \CC 0.729 \\
\midrule
\multirow{5}{*}{USM}          & Pretrained & 0.921                   & 0.921                          & 0.921                  & 0.879                   & 0.879                          & 0.879                  & 0.748                              \\
                              & BoN        & 0.895             & 0.890                   & 0.912           & 0.862             & 0.866                    & 0.841            & 0.749                       \\
                              & SMC        & 0.923                   & 0.920                          & 0.927                  & 0.874                    & 0.875                          & 0.874                  & 0.748                              \\
                              & SVDD       & 0.922                   & 0.918                          & 0.921                  & 0.875                    & 0.875                          & 0.875                  & 0.748                              \\
                               & \CC \textbf{\ourmethod}         & \CC 0.880          & \CC 0.827                 & \CC 0.868         & \CC 0.856          & \CC 0.837                 & \CC 0.841         & \CC 0.734                     \\
                               & \CC \textbf{\ourmethod-B} 
                               & \CC 0.881 & \CC 0.856 & \CC 0.881  
                               & \CC 0.853 & \CC 0.861 & \CC 0.838 
                               & \CC 0.732 \\
\midrule
\multirow{5}{*}{SEDD-M}    & Pretrained & 0.926                   & 0.926                          & 0.926                  & 0.905                   & 0.905                          & 0.905                  & 0.748                              \\
                           & BoN        & 0.903             & 0.889                    & 0.926            & 0.874            & 0.874                   & 0.877            & 0.749                        \\
                           & SMC        & 0.925 & 0.914 & 0.927 & 0.902 & 0.902 & 0.902 & 0.748 \\
                           & SVDD       & 0.928 & 0.928 & 0.928 & 0.907 & 0.907 & 0.907 & 0.748 \\
                           & \CC \textbf{\ourmethod}         & \CC 0.885          & \CC 0.847                 & \CC 0.890         & \CC 0.865         & \CC 0.880                & \CC 0.837         & \CC 0.749                     \\
                           & \CC \textbf{\ourmethod-B} 
                           & \CC 0.898 & \CC 0.858 & \CC 0.898  
                           & \CC 0.887 & \CC 0.888 & \CC 0.870
                           & \CC 0.732 \\
\midrule
\multirow{6}{*}{SEDD-U} & Pretrained & 0.922                   & 0.922                          & 0.922                  & 0.879                   & 0.879                          & 0.879                  & 0.631                              \\
                        & BoN        & 0.900             & 0.900                    & 0.915            & 0.866             & 0.867                   & 0.851            & 0.661                              \\
                        & SMC        & 0.901                   & 0.905                          & 0.910                  & 0.872                   & 0.873                          & 0.866                  & 0.661                                \\
                        & SVDD       & 0.906                   & 0.907                          & 0.914                  & 0.873                    & 0.877                          & 0.870                  & 0.666                                \\ 
                        & SGDD       & 0.906                   & 0.913                          & 0.908                  & 0.868                   & 0.866                          & 0.859                  & 0.650                       \\
                        & \CC \textbf{\ourmethod}         & \CC 0.880          & \CC 0.864               & \CC 0.863        & \CC 0.855          & \CC 0.868                 & \CC 0.825        & \CC 0.709 \\
                        & \CC \textbf{\ourmethod-B} 
                        & \CC 0.869 & \CC 0.865  & \CC 0.848 
                        & \CC 0.859 & \CC 0.863 & \CC 0.843
                        & \CC 0.662 \\
\bottomrule
\end{tabularx}
\caption{\textbf{Diversity metrics for each method and reward.} Higher is better.}
\vspace{-\baselineskip}
\label{tab:diversity}
\end{table}
\vspace{-\baselineskip}
\subsection{Wall-Clock Time Comparisons} 
\label{app:wall-clock}
We provide additional wall-clock time comparisons in Tab.~\ref{tab:app-zinc250k-matched-wall-clock}. We also provide a comparison of the wall-clock times for \ourmethod{} and \ourmethod-B in Tab.~\ref{tab:app-wall-clock-ours}, showing that \ourmethod-B provides significantly accelerated sampling compared to \ourmethod{}.
\vspace{-0.5\baselineskip}
\begin{table}[ht]
{
\scriptsize
\begin{tabularx}{\linewidth}{Y|Y Y}
\toprule
             & \multicolumn{1}{Y}{\ourmethod} & \multicolumn{1}{Y}{\ourmethod-B} \\
\midrule
Time (s)     & 3029                     & 334 \\
Batch Size   & 1                        & 8 \\
NFE & 1024                     & 1024  \\
Reward       & 0.916                    & 0.917  \\
\bottomrule
\end{tabularx}
}
\caption{\textbf{Wall-clock time comparisons.} Wall-clock time measured \ourmethod{} applied to USM guided by ZINC250K QED reward. Batching significantly improves the speed of \ourmethod.}
\vspace{-\baselineskip}
\label{tab:app-wall-clock-ours}
\end{table}
\begin{table}[h]
{
\scriptsize
\begin{tabularx}{\linewidth}{>{\centering\arraybackslash}m{0.06\textwidth}|YYYYY|YYYYY|YYYYY}
             \toprule
             & \multicolumn{15}{c}{ZINC250K}                                                                                          \\
             \midrule
             & \multicolumn{5}{c|}{QED}                & \multicolumn{5}{c|}{Rings}             & \multicolumn{5}{c}{SA}                \\
             & \ourmethod  & BoN   & BoN    & SMC   & SVDD  & \ourmethod  & BoN   & BoN   & SMC   & SVDD  & \ourmethod  & BoN   & BoN   & SMC   & SVDD  \\
             \midrule
Time (s)     & 334   & 334   & 359    & 359   & 264   & 321   & 331   & 359   & 320   & 276   & 337   & 328   & 300   & 341   & 279   \\
Batch Size   & 8     & 78    & 8      & 14    & 886   & 8     & 78    & 8     & 14    & 886   & 8     & 78    & 8     & 14    & 886   \\
NFE & 1024  & 5750  & 1024   & 1024  & 65536 & 1024  & 5750  & 1024  & 1024  & 65536 & 1024  & 5750  & 1024  & 1024  & 65536 \\
Reward       & 0.917 & 0.934 & 0.9133 & 0.766 & 0.719 & 5.315 & 4.945 & 4.312 & 2.630 & 2.607 & 0.912 & 0.919 & 0.892 & 0.792 & 0.755 \\
\bottomrule
\end{tabularx}
}
\caption{\textbf{USM ZINC250K rewards under matched wall-clock time.} Rewards and time are computed from 128 drawn samples. \ourmethod-B significantly outperforms all other baselines in ring count while being comparable with BoN in QED and SA.}
\vspace{-\baselineskip}
\label{tab:app-zinc250k-matched-wall-clock}
\end{table}

\subsection{Autocorrelation Plots and Acceptance Rate}
\label{subapp:acceptance-rate}
We provide autocorrelation plots for ZINC250K MDLM and USM in Fig.~\ref{fig:app-acf-zinc} to further investigate the burn-in time. The autocorrelation function quickly vanishes to zero within the first 2000 iterations and remains small until the end, indicating that the chain converges quickly. Since these metrics are only heuristics for measuring convergence, we conservatively discard the first half of the chain. However, the autocorrelation plots suggest that it may be possible to burn fewer samples.

\begin{figure}[h]
    \centering
    \includegraphics[width=\linewidth]{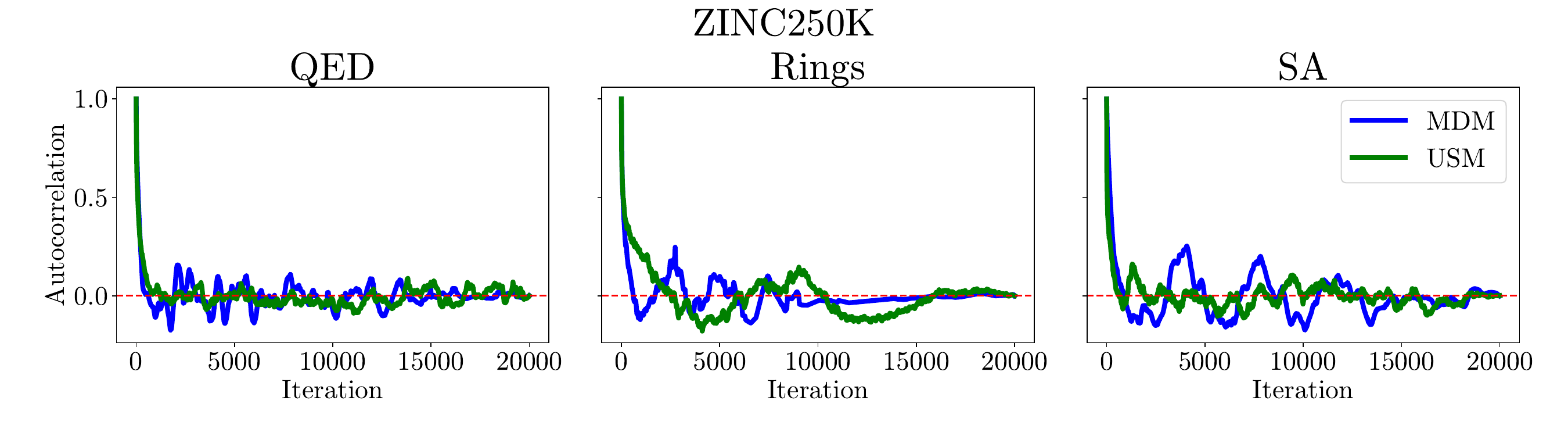}
    \caption{\textbf{Autocorrelation plots for ZINC250K MDM and USM.}}
    \label{fig:app-acf-zinc}
\end{figure}

We provide acceptance rates for \ourmethod{} \ with MDLM and USM in Tab.~\ref{tab:acceptance-rate}. Although the acceptance rate for MDLM is quite low due to MDLM often generating invalid molecules, the acceptance rate is still acceptable as we draw 128 samples from ~13k samples after the burn-in period. On the other hand, USM has much higher acceptance rates which can be attributed to its higher validity rate due to the ability to correct errors from previous diffusion steps
\begin{table}[h]
\centering
\begin{tabularx}{\textwidth}{Y|YYY|YYY|Y}
\toprule
                          & \multicolumn{3}{c|}{QM9} & \multicolumn{3}{c|}{ZINC250K} & MPRA  \\
                          & QED    & Rings  & SA     & QED      & Rings    & SA      & HepG2 \\
                          \midrule
\multicolumn{1}{c|}{MDM} & 0.032  & 0.002  & 0.009  & 0.009    & 0.003    & 0.011   & 0.029 \\
\multicolumn{1}{c|}{USM}  & 0.510  & 0.330  & 0.346  & 0.725    & 0.597    & 0.711   & 0.038 \\
\bottomrule
\end{tabularx}
\caption{\textbf{Acceptance rates for MDM and USM.}}
\label{tab:acceptance-rate}
\end{table}

\newpage
\subsection{NFE Comparison}
\label{app:nfe}
In this section, we compare the performance of each method across various NFEs for QM9 and ZINC250K. We use NFEs $\in \{512, 1024, 2048, 4096\}$ for molecule tasks and NFEs $\in \{ 500, 1000, 2000, 4000 \}$ for MPRA. The results are shown in Fig.~\ref{fig:nfes}. As shown in the figure, \ourmethod{} consistently and significantly outperforms all other methods in ring count and HepG2 activity across all diffusion models and datasets. \ourmethod{} also outperforms all other methods in SA except for BoN when using USM on ZINC250K, and outperforms all other methods in QED except for BoN when using USM. We note that for the hardest reward ring count where the high-reward samples lie in extremely low density regions, \ourmethod{} consistently outperforms all other methods by a large margin.

\begin{figure}[htbp]
    \centering
    \includegraphics[width=\linewidth]{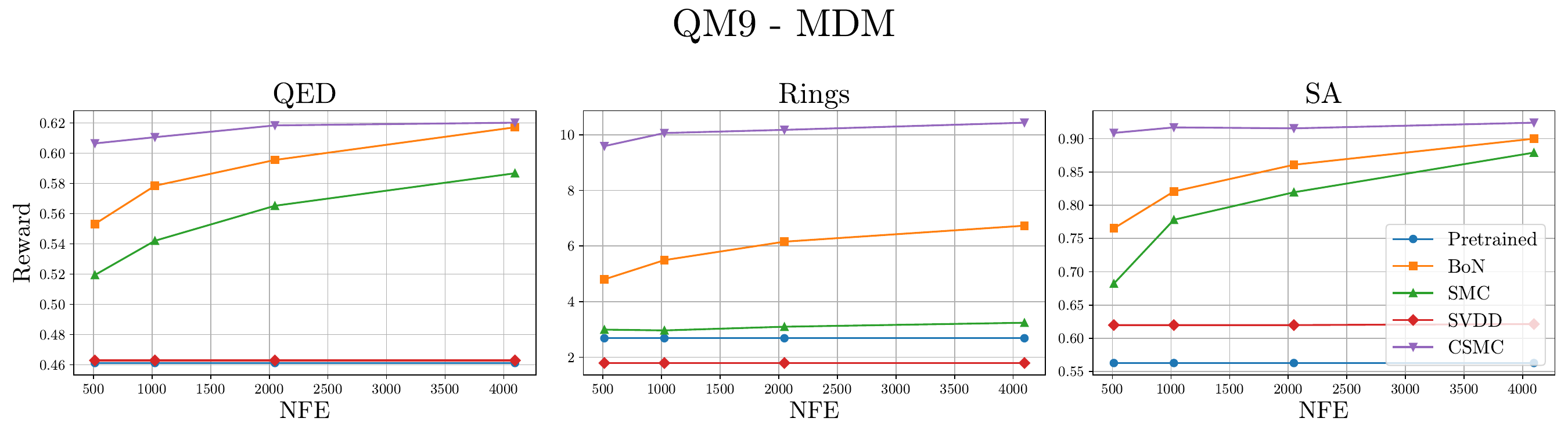}
    \includegraphics[width=\linewidth]{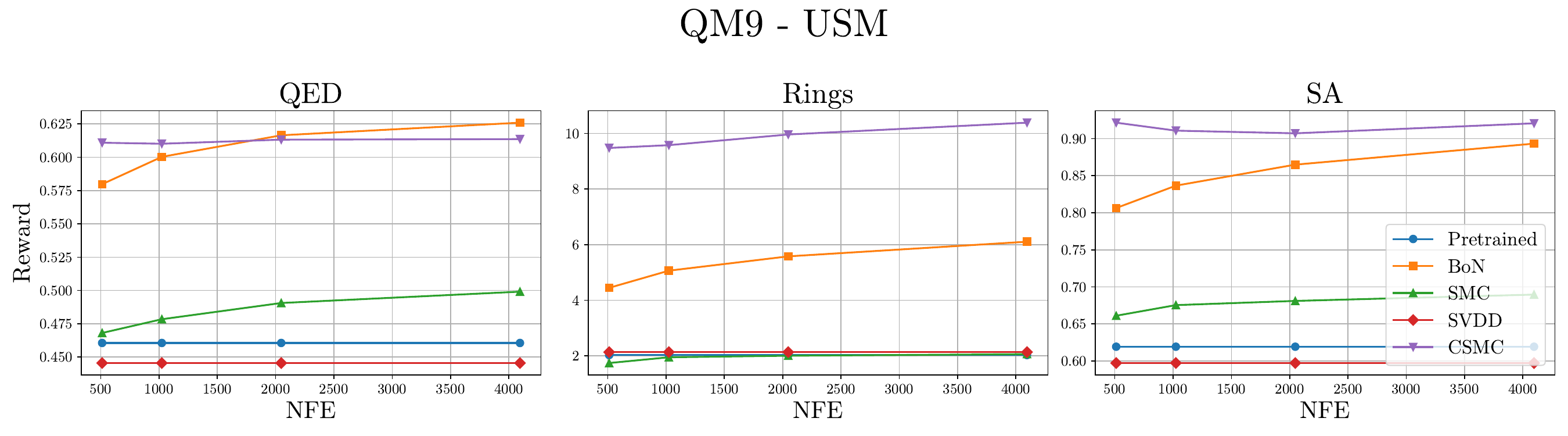}
    \includegraphics[width=\linewidth]{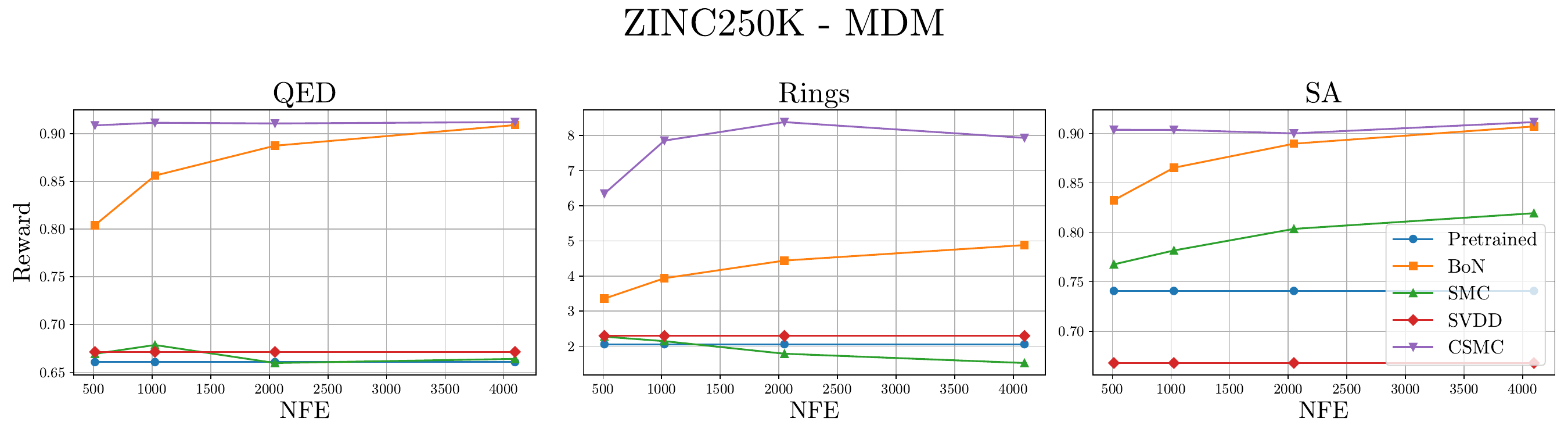}
    \includegraphics[width=\linewidth]{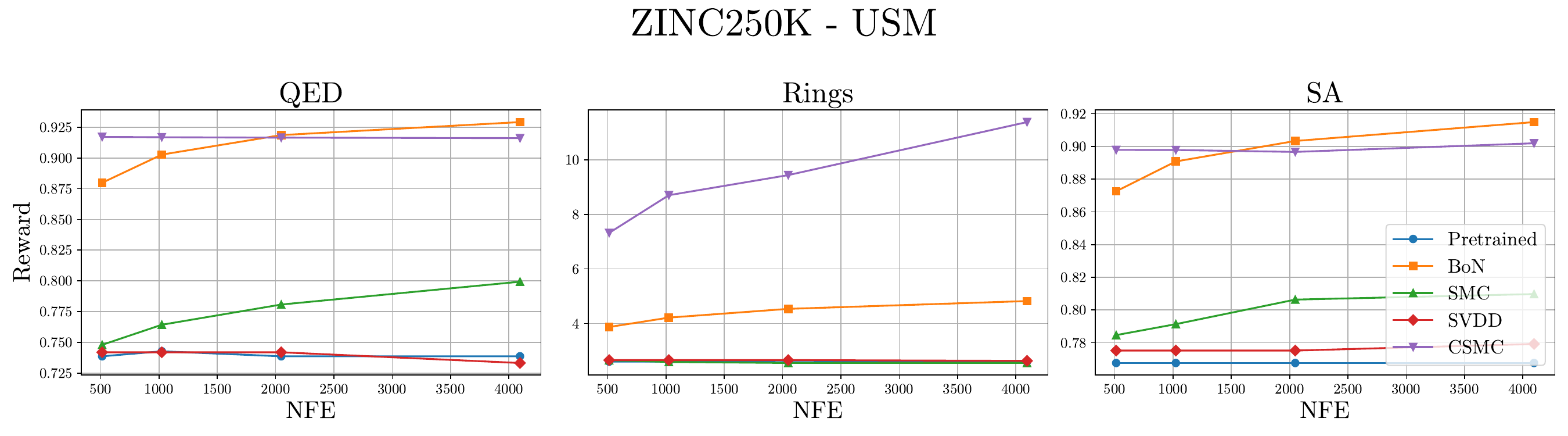}
    \includegraphics[width=0.48\linewidth]{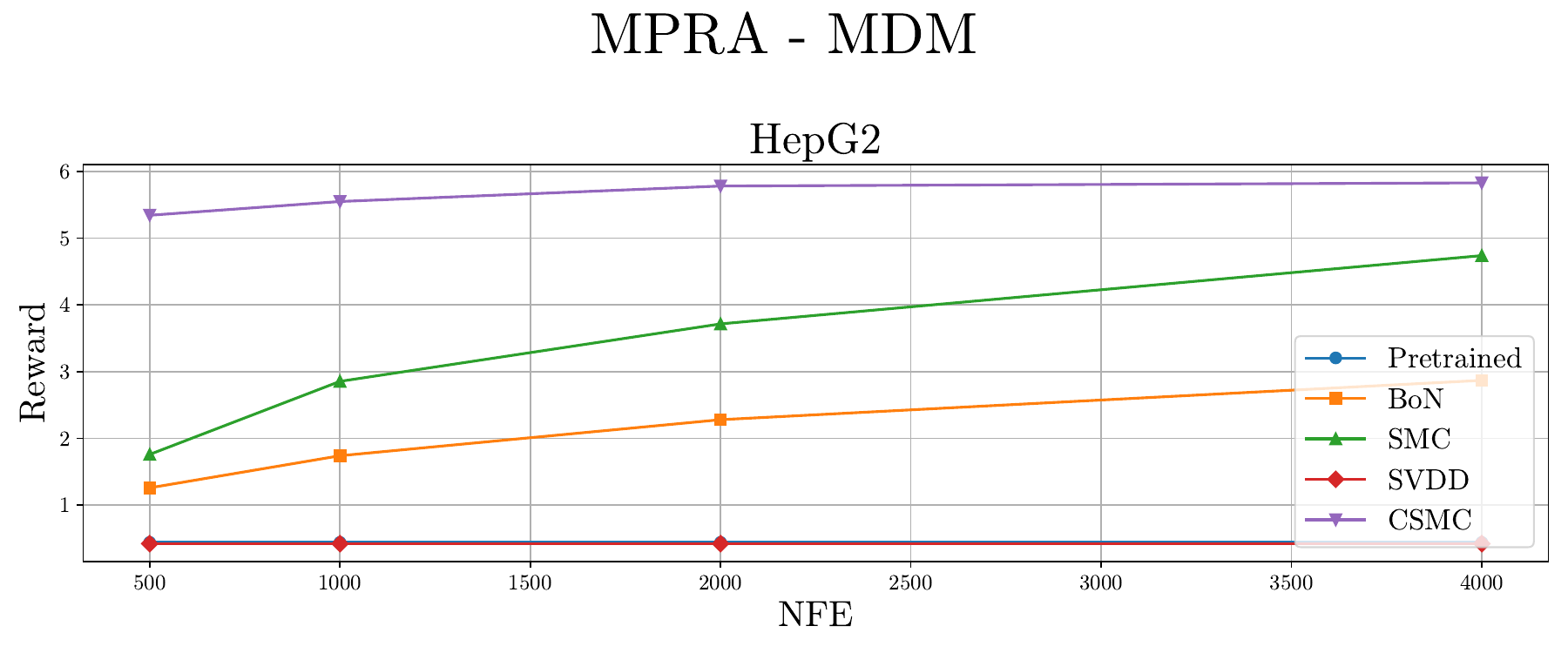}
    \includegraphics[width=0.48\linewidth]{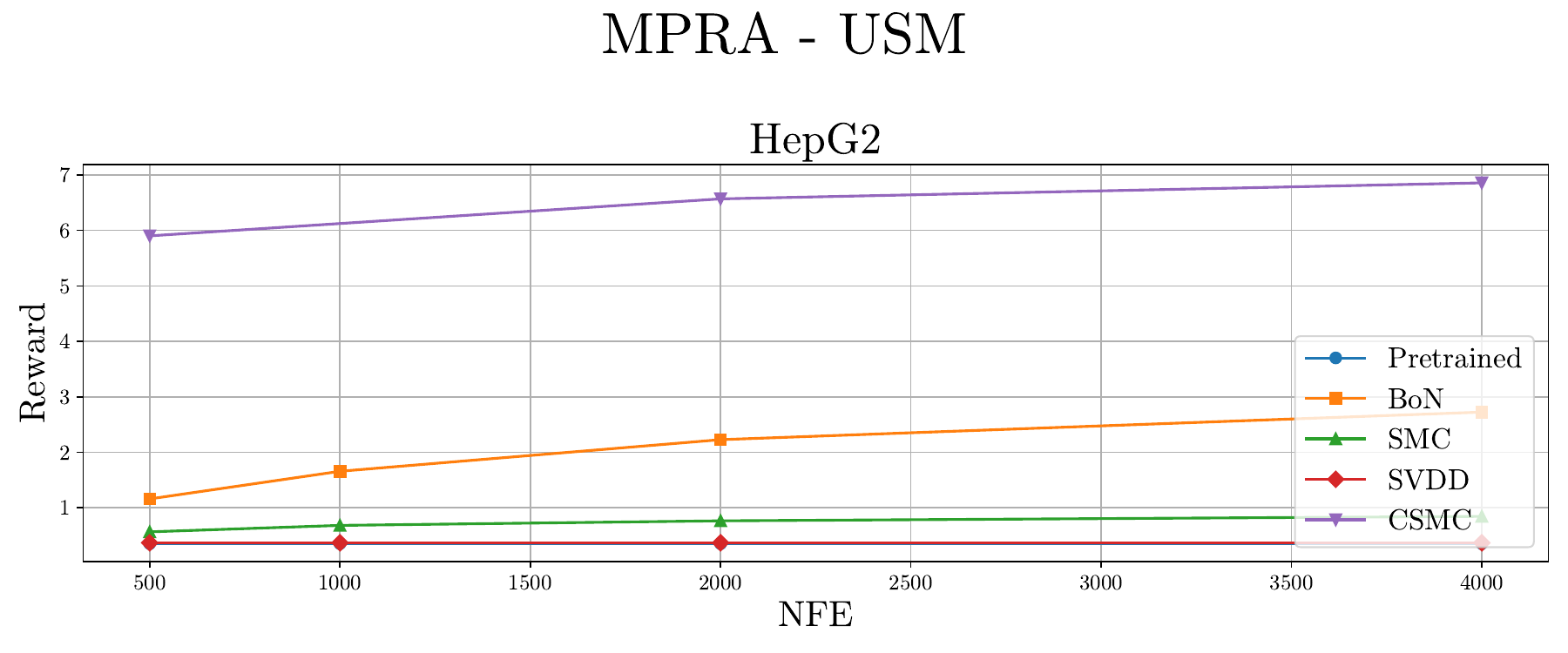}
    \caption{\textbf{Performance across various NFEs.} We run each method on molecule generation using NFEs $\in \{512, 1024, 2048, 4096\}$ for molecule tasks and NFE $\in \{ 500, 1000, 2000, 4000 \}$ for MPRA.}
    \label{fig:nfes}
\end{figure}

\subsection{Time Parameters $t_\text{lo}$ and $t_\text{hi}$}
\label{subapp:time-params}
We provide additional analysis on the time parameters $t_\text{lo}$ and $t_\text{hi}$. Larger $t_\text{lo}$ and $t_\text{hi}$ lead to increased exploration as the samples from the proposal distribution become more uncorrelated due to the larger noise scale. Smaller values lead to increased exploitation as samples are more correlated due to the smaller noise scale. 

To test the sensitivity of \ourmethod{} with respect to these time parameters, we fix $t_\text{lo}=0.2$ and vary $t_\text{hi} \in [0.2,0.7]$. We plot the reward and diversity for ZINC250K USM samples with respect to these parameters. We also include a plot where we fix $t_\text{hi}=0.8$ and vary $t_\text{lo} \in [0.3,0.8]$. The results are shown in Fig.~\ref{fig:t-ablation}. As shown in the figure, our method is not sensitive to the time parameters in general. However, For ring count, the reward decreases as $t_\text{lo}$ increases. Intuitively, as $t_\text{lo}$ increases, each \ourmethod{} \ step changes the current $x_0$ sample significantly, resulting in greater exploration but less exploitation. For simpler rewards such as QED and SA, this does not significantly impact the average reward. However, for ring count where high-reward samples lie in extremely low density regions of the pretrained model's learned distribution (refer to Fig.~\ref{fig:reward-distribution}), it is necessary to choose smaller times to exploit and explore around the current $x_0$.

\begin{figure}[htbp]
    \centering
    \includegraphics[width=\linewidth]{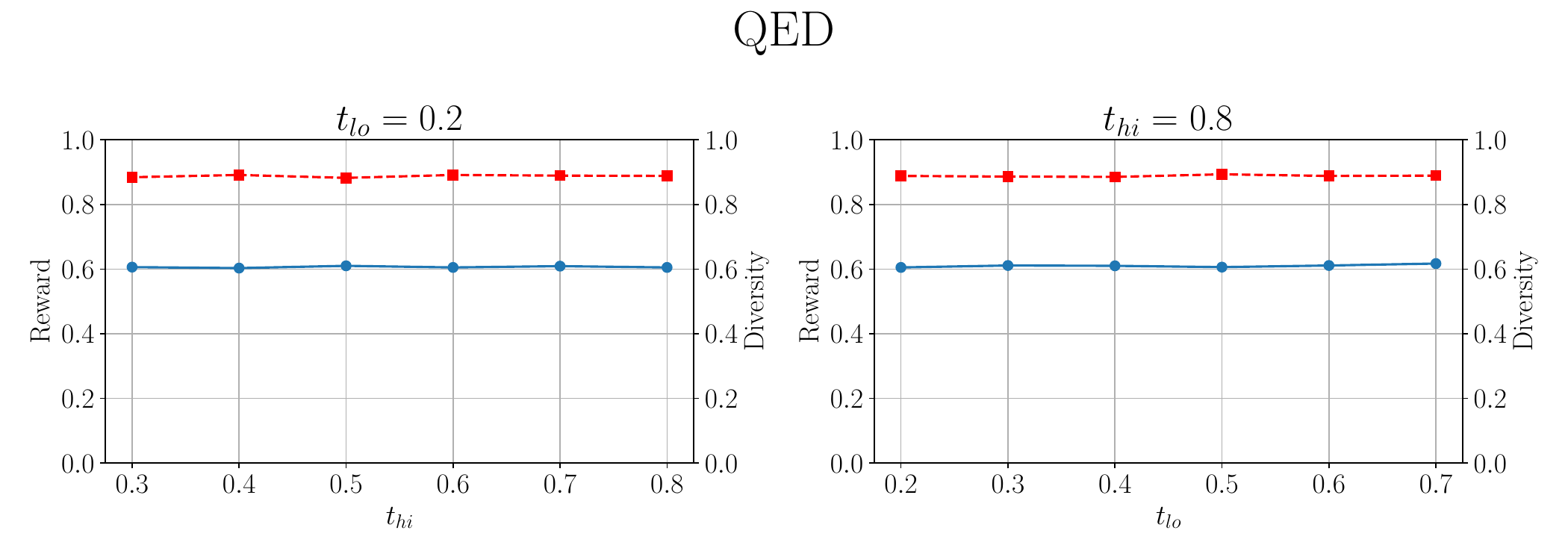}
    \includegraphics[width=\linewidth]{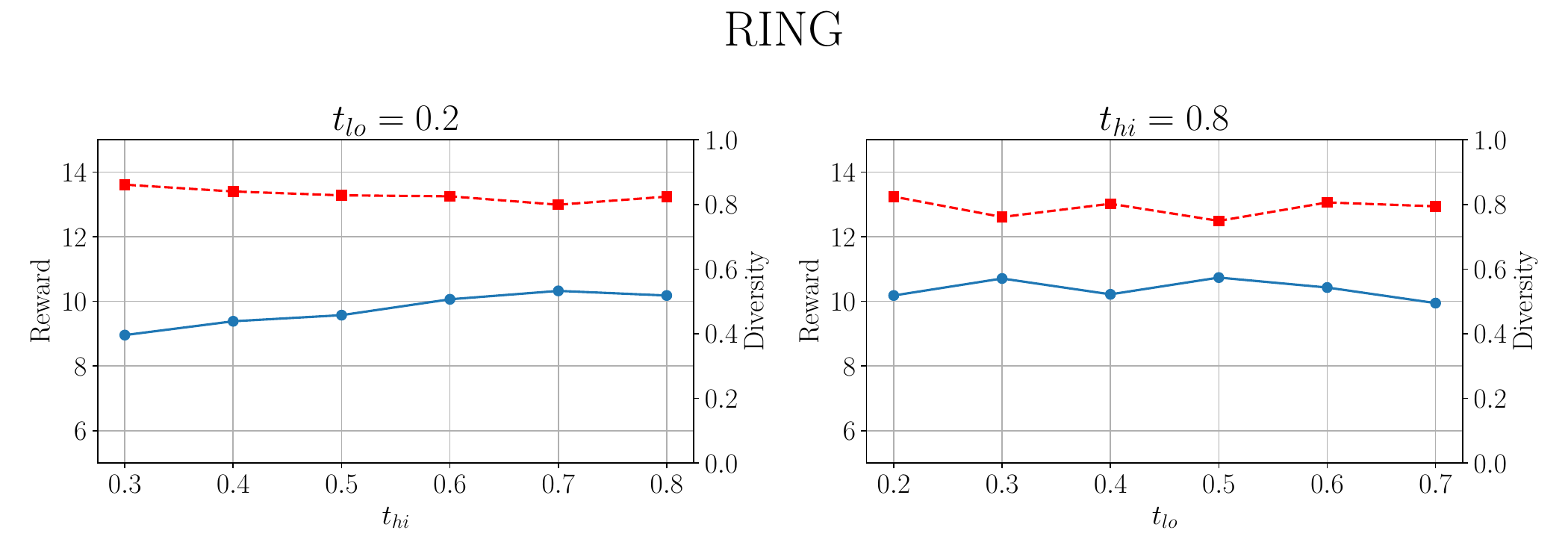}
    \includegraphics[width=\linewidth]{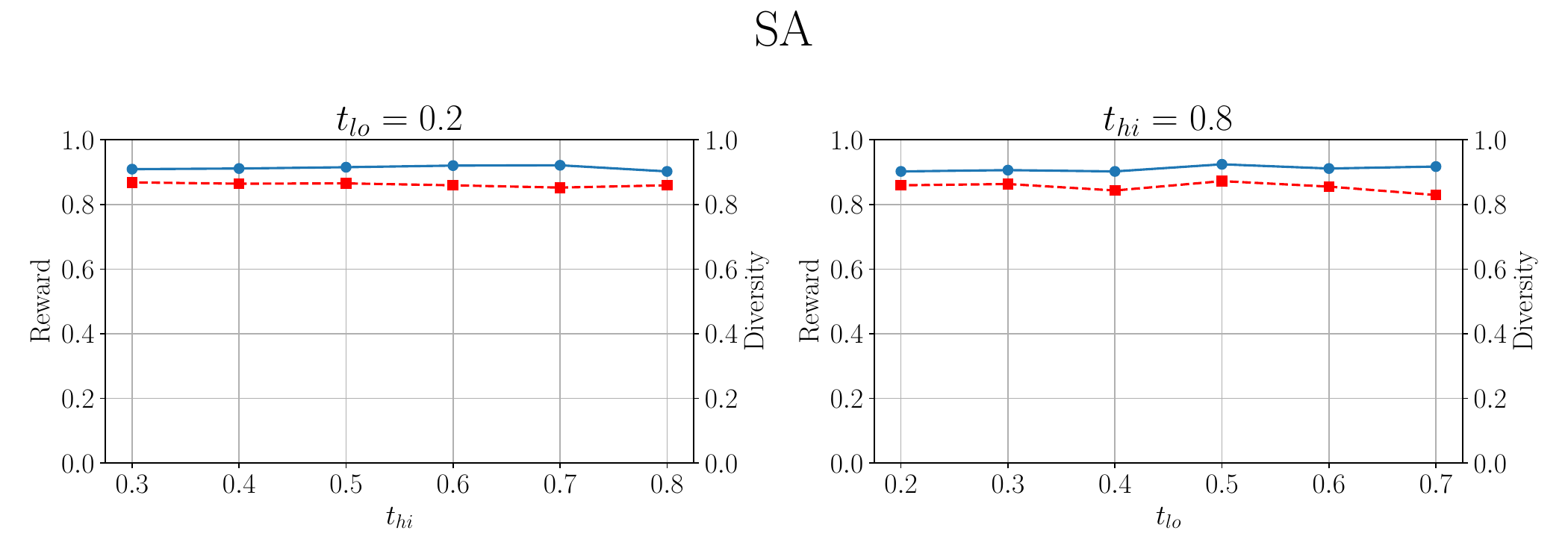}
    \caption{\textbf{Analysis of time parameters.} We plot the reward and diversity for ZINC250K USM samples with respect to (1) fixing $t_\text{lo}=0.2$ and vary $t_\text{hi} \in [0.3,0.8]$, and (2) fixing $t_\text{hi}=0.8$ and vary $t_\text{lo} \in [0.2,0.7]$.}
    \label{fig:t-ablation}
\end{figure}
\section{Qualitative Results}
\label{app:qualitatives}
 Qualitative results are shown in Fig.~\ref{fig:qualitative-mol}.

\begin{figure}[h!]

\centering
\setlength{\tabcolsep}{0.08em}
\renewcommand{\arraystretch}{1.0}
{\small
\begin{tabular}{@{}
  >{\centering\arraybackslash}m{0.14\linewidth} |
  >{\centering\arraybackslash}m{0.42\linewidth} |
  >{\centering\arraybackslash}m{0.42\linewidth}
@{}}
\toprule
& \textbf{Pretrained} & \textbf{\ourmethod} \\
\midrule

\multirow{3}{*}{\textbf{QM9}} &
\begin{tabular}{@{}c c c@{}}
  \makebox[0.13\textwidth]{QED: 0.378} &
  \makebox[0.13\textwidth]{Rings: 1} &
  \makebox[0.13\textwidth]{SA: 0.501} \\
  \includegraphics[width=0.13\textwidth]{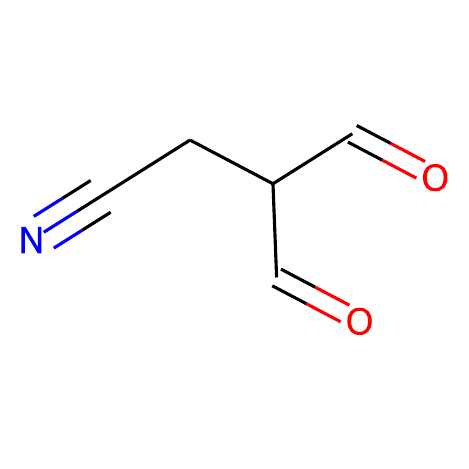} &
  \includegraphics[width=0.13\textwidth]{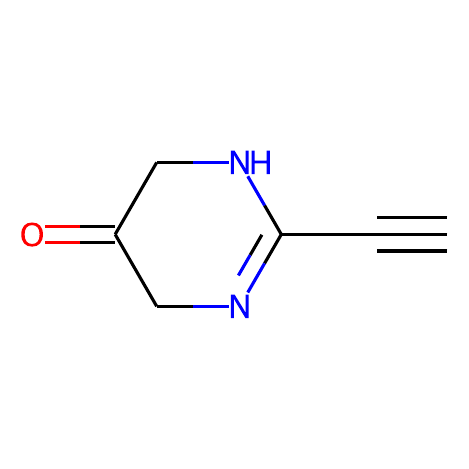} &
  \includegraphics[width=0.13\textwidth]{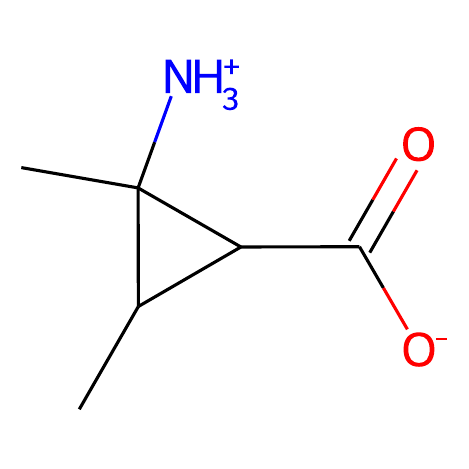} \\
\end{tabular}
&
\begin{tabular}{@{}c c c@{}}
  \makebox[0.13\textwidth]{QED: 0.578} &
  \makebox[0.13\textwidth]{Rings: 9} &
  \makebox[0.13\textwidth]{SA: 0.920} \\
  \includegraphics[width=0.13\textwidth]{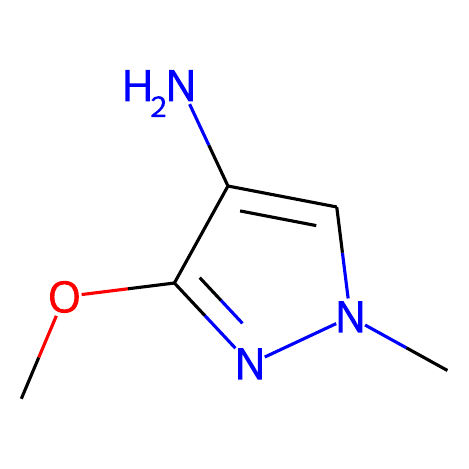} &
  \includegraphics[width=0.13\textwidth]{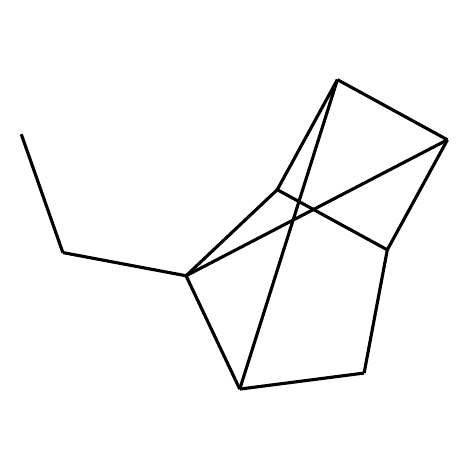} &
  \includegraphics[width=0.13\textwidth]{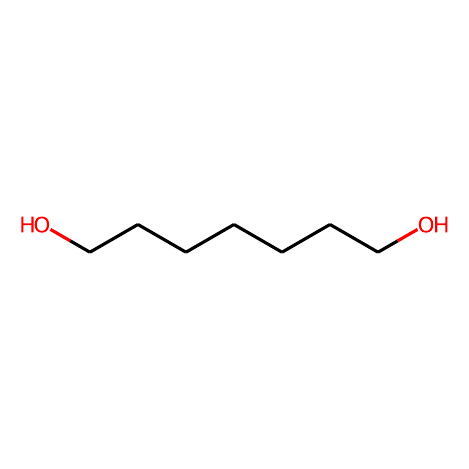} \\
\end{tabular}
\\
\midrule

\multirow{3}{*}{\textbf{ZINC250K}} &
\begin{tabular}{@{}c c c@{}}
  \makebox[0.13\textwidth]{QED: 0.610} &
  \makebox[0.13\textwidth]{Rings: 2} &
  \makebox[0.13\textwidth]{SA: 0.802} \\
  \includegraphics[width=0.13\textwidth]{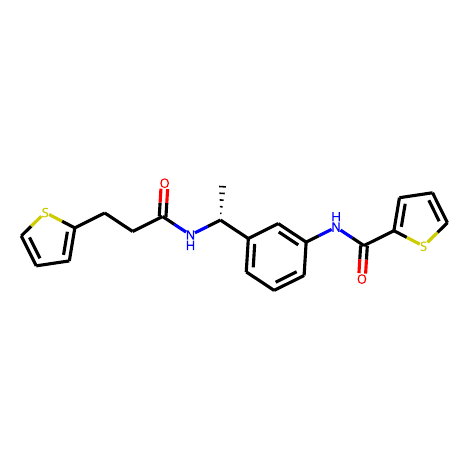} &
  \includegraphics[width=0.13\textwidth]{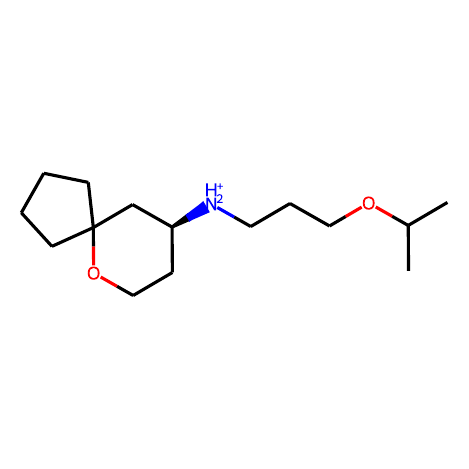} &
  \includegraphics[width=0.13\textwidth]{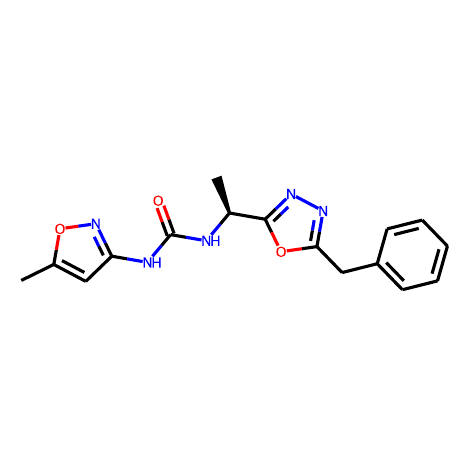} \\
\end{tabular}
&
\begin{tabular}{@{}c c c@{}}
  \makebox[0.13\textwidth]{QED: 0.934} &
  \makebox[0.13\textwidth]{Rings: 7} &
  \makebox[0.13\textwidth]{SA: 0.915} \\
  \includegraphics[width=0.13\textwidth]{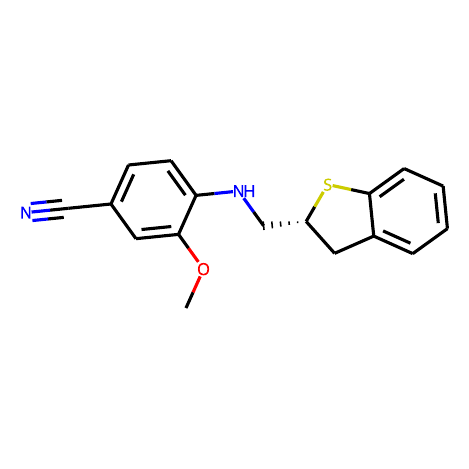} &
  \includegraphics[width=0.13\textwidth]{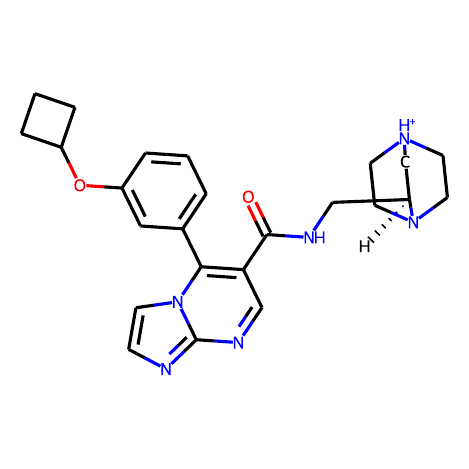} &
  \includegraphics[width=0.13\textwidth]{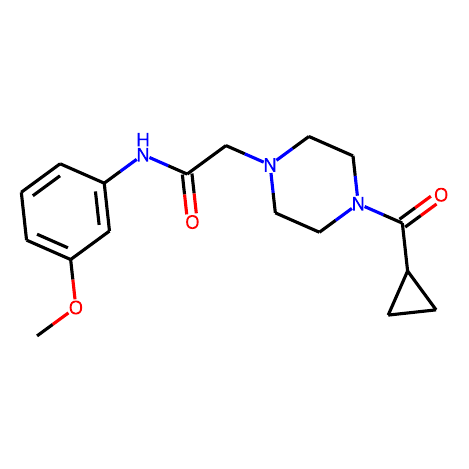} \\
\end{tabular}
\\
\bottomrule
\end{tabular}
}

\caption{\textbf{Qualitative results.} Molecules randomly sampled from discrete diffusion models pretrained on QM9~\citep{ramakrishnan2014qm9} and ZINC250K~\citep{irwin2012zinc}. \ourmethod \ generates molecules with high rewards as shown on the right.}
\vspace{-\baselineskip}
\label{fig:qualitative-mol}
\end{figure}


\end{document}